\documentclass[twoside]{article}

\usepackage[accepted]{aistats2022}
\usepackage[round]{natbib}

\usepackage{xr}

\usepackage{booktabs}
\usepackage[usenames,dvipsnames]{xcolor}
\usepackage{latexsym}              
\usepackage{amsmath}               
\usepackage{amssymb}               
\usepackage{amsfonts}              
\usepackage{amsthm}                
\usepackage{accents}
\usepackage{mathtools}
\usepackage{multirow}
\usepackage{tikz}                  
\usetikzlibrary{arrows,positioning,shapes}
\usepackage{pifont}
\usepackage[ruled,vlined]{algorithm2e}
\usepackage{enumitem}
\usetikzlibrary{matrix}

\usepackage[mathcal]{eucal}
\usepackage{breakcites}

\usepackage{cleveref}
\crefname{assumption}{Assumption}{Assumptions}
\crefname{equation}{Eq.}{Eqs.}
\crefname{figure}{Fig.}{Figs.}
\crefname{table}{Table}{Tables}
\crefname{section}{Sec.}{Secs.}
\crefname{algorithm}{Algorithm}{Algorithms}
\crefname{theorem}{Thm.}{Thms.}
\crefname{lemma}{Lemma}{Lemmas}
\crefname{proposition}{Prop.}{Propositions}
\crefname{corollary}{Cor.}{Cors.}
\crefname{example}{Example}{Examples}
\crefname{appendix}{Appendix}{Appendixes}
\crefname{remark}{Remark}{Remark}

\newcounter{remark}[section]

\newcommand{\MM}{\operatorname{MM}}
\newcommand{\M}{\operatorname{M}}
\newcommand{\RM}{\operatorname{RM}}
\newcommand{\F}{\operatorname{F}}

\newcommand{\rank}{\operatorname{rank}}

\newcommand{\calA}{\mathcal{A}}
\newcommand{\calV}{\mathcal{V}}
\newcommand{\calX}{\mathcal{X}}
\newcommand{\calY}{\mathcal{Y}}

\newcommand{\calZ}{\mathcal{Z}}
\newcommand{\calQ}{\mathcal{Q}}
\newcommand{\calE}{\mathcal{E}}

\newcommand{\calP}{\mathcal{P}}

\newcommand{\calG}{\mathcal{G}}

\newcommand{\calC}{\mathcal{C}}

\newcommand{\calN}{\mathcal{N}}

\newcommand{\I}{\operatorname{I}}
\newcommand{\II}{\operatorname{II}}
\newcommand{\III}{\operatorname{III}}
\newcommand{\IV}{\operatorname{IV}}

\DeclareMathAlphabet{\mathsfsl}{OT1}{cmss}{m}{sl}

\renewcommand{\phi}{\varphi}

\newcommand{\Rspace}[1]{\mathbb{R}^{#1}}

\newcommand{\argmin}{\operatorname*{arg\; min}}
\newcommand{\argmax}{\operatorname*{arg\; max}}

\newcommand{\Expect}{\operatorname{\mathbb{E}}}

\theoremstyle{plain}  
\newtheorem{theorem}{Theorem}[section]
\newtheorem{definition}[theorem]{Definition}

\newtheorem{lemma}[theorem]{Lemma}
\newtheorem{proposition}[theorem]{Proposition}

\newtheorem{corollary}[theorem]{Corollary}

\begin{document}

%

%

\twocolumn[

\aistatstitle{On the Consistency of Max-Margin Losses}

\aistatsauthor{ Alex Nowak-Vila \And Alessandro Rudi \And  Francis Bach }

\aistatsaddress{ ENS-INRIA-PSL \\ Paris, France \And  ENS-INRIA-PSL \\ Paris, France \And ENS-INRIA-PSL \\ Paris, France } ]

\begin{abstract}
The foundational concept of Max-Margin in machine learning is ill-posed for output spaces with more than two labels such as in structured prediction. In this paper, we show that the Max-Margin loss can only be consistent to the classification task under highly restrictive assumptions on the discrete loss measuring the error between outputs. These conditions are satisfied by distances defined in tree graphs, for which we prove consistency, thus being the first losses shown to be consistent for Max-Margin beyond the binary setting. We finally address these limitations by correcting the concept of Max-Margin and introducing the Restricted-Max-Margin, where the maximization of the loss-augmented scores is maintained, but performed over a subset of the original domain. The resulting loss is also a generalization of the binary support vector machine and it is consistent under milder conditions on the discrete loss.
\end{abstract}

\section{INTRODUCTION} \label{sec:introduction}
One of the first binary classification methods learned in a machine learning course is the support vector machine (SVM) \citep{boser1992training, cortes1995support} and it is introduced using the principle of maximum margin: assuming the data are linearly separable, the classification hyperplane must maximize the separation to the observed examples. Having this intuition in mind, the same principle has been used to extend this notion to larger output spaces $\calY$, such as multi-class classification \citep{crammer2001algorithmic} and structured prediction \citep{taskar2004max,tsochantaridis2005large}, where the separation to the observed examples is controlled by a discrete loss~$L(y, y')$ measuring the error between outputs $y$ and~$y'$. The resulting method generalizes the binary SVM and corresponds to minimizing the so-called Max-Margin loss
\begin{equation}\label{eq:maxmarginloss}
    S_{\M}(v, y) = \max_{y'\in\calY}~L(y, y') + v_{y'} - v_y,
\end{equation}
where $v\in\Rspace{|\calY|}$ is a vector with coordinate $v_y$ encoding the score for output $y$.
Unfortunately, this method may not be \emph{consistent}, i.e., minimizing the Max-Margin loss \eqref{eq:maxmarginloss} may not lead to a minimization of the discrete loss~$L$ of interest. In particular, it is known that the Max-Margin loss is only consistent for the 0-1 loss under the dominant label condition, i.e., when for every input there exists an output element with probability larger than $1/2$~\citep{liu2007fisher}, which is always satisfied in the binary case. However, far less is known for other tasks. Indeed, the Max-Margin loss is widely used for structured output spaces where the discrete loss $L$ defining the task is different than the 0-1 loss, under the name of Structural SVM (SSVM) \citep{taskar2005discriminative, caetano2009learning, smith2011linguistic} or Max-Margin Markov Networks ($\operatorname{M^3N}$)~\citep{taskar2005discriminative}. In this general setting, the following questions remain unanswered:
\begin{itemize}
    \item[(i)] \emph{Does there exist a necessary condition on $L$ for consistency to hold? Does it exist a space of losses for which consistency holds? Can we generalize the consistency result under the dominant label assumption beyond the 0-1 loss?}
    \item[(ii)] \emph{Can we correct the Max-Margin loss to make it consistent by maintaining the additive and maximization structure of the Max-Margin loss?}
\end{itemize}
We answer these questions in this paper. In particular, we make the following contributions:
\begin{itemize}
    \item[-] We prove that the Max-Margin loss can only be consistent under a restrictive necessary condition on the structure of the loss $L$, indeed, the loss $L$ has to be a distance and satisfy the triangle inequality as an \emph{equality} for several groups of outputs. As a positive result, we show that a distance defined in a tree graph, such as the absolute deviation loss used in ordinal regression, satisfies this condition and it is consistent, thus providing the first set of losses for which consistency holds beyond the binary setting. We also extend the existing \emph{partial} consistency result of the 0-1 loss by extending the result under the dominant label condition to all losses that are distances.
    \item[-] As a secondary contribution, we introduce the Restricted-Max-Margin loss, where the maximization of the loss-augmented scores defining the Max-Margin loss is restricted to a subset of the simplex. The resulting loss also generalizes the binary SVM and it is consistent under milder assumptions on $L$. Moreover, we show the connections between these losses and the Max-Min-Margin loss \citep{fathony2016adversarial, duchi2018multiclass, nowak2020consistent}, where consistency always holds independently of the discrete loss $L$.
\end{itemize}

\section{MAX-MARGIN AND MAIN RESULTS}\label{sec:2} 

In this section we introduce the concept of Max-Margin learning and its consistency from its origins in binary classification to the structured output setting. This is followed by the presentation of the main results of this paper and its implications are discussed.

\subsection{Max-Margin Learning}
\paragraph{Binary output.}
Let~$(x_1, y_1), \ldots, (x_n, y_n)$ be $n$ examples of input-output pairs sampled from an unknown distribution $\rho$ defined in~$\calX\times\calY$. Let us first assume that the input space $\calX$ is a vector space and~$\calY=\{-1, 1\}$ represents binary labels. The goal is to construct a binary-valued function~$f:\calX\xrightarrow[]{}\calY$ minimizing the expected classification error 
\begin{equation}\label{eq:discreterisk}
    \calE(f) = \Expect_{(x,y)\sim\rho}L(f(x), y),
\end{equation}
where $L(y, y') = 1(y\neq y')$ is the binary 0-1 loss. The concept of max-margin was initially defined in this setting to construct a predictor of the form $\operatorname{sign}(g(x))$ where~$g(x) = w^\top x + b$ is an affine function defining a hyperplane with maximum separation to the examples assuming linearly separable data \citep{boser1992training}. In this setting, an example~$x_i$ is correctly classified if $(w^\top x_i + b_i)y_i > 0$ and misclassified otherwise. The max-margin hyperplane can be found by minimizing $\|w\|_2^2$ under the constraint~$(w^\top x_i + b)y_i \geq 1$ for all $n$ examples. When the data are not linearly separable, some examples are allowed to be misclassified by introducing some non-negative \emph{slack variables} $\xi_i$ and solving the optimization problem known as the {\em support vector machine (SVM)}~\citep{cortes1995support}:
\begin{equation*}
\left\{\begin{array}{ll}
     \underset{w,b,\xi}{\min} & \frac{1}{n}\sum_{i=1}^n\xi_i + \frac{\lambda}{2}\|w\|_2^2  \\
     \text{s.t.} & (w^\top x_i + b)y_i \geq 1 - \xi_i, \hspace{0.4cm} \xi_i\geq 0,
        ~\forall i\in[n],
\end{array}\right.
\end{equation*}
where $\lambda>0$ is a parameter used to balance the first term with the second. We can re-write the constraints as~$\xi_i\geq 1 - y_ig(x_i)$ for non-negative $\xi_i$'s and extend the affine hypothesis space to a generic functional space $\calG$ with associated norm $\|\cdot\|_{\calG}$ to allow for non-linear predictors, such as reproducing kernel Hilbert spaces (RKHS)~\citep{aronszajn1950theory}. Then, the problem above can be written as a convex regularized empirical risk minimization (ERM)~\citep{vapnik1992principles} problem
\begin{equation}\label{eq:maxmarginERM}
    \min_{g\in\calG}~\frac{1}{n}\sum_{i=1}^nS_{\M}(g(x_i), y_i) + \frac{\lambda}{2}\|g\|_{\calG}^2,
\end{equation}
where $S_{\M}(v, y) = \max(1-yv, 0)$ is the binary Max-Margin loss (also called SVM loss), and now~$\lambda$ can be interpreted as the regularization parameter. An important property of the classification method is that the estimated predictor solving \eqref{eq:maxmarginERM} over all measurable functions converges to the predictor $f^\star$ minimizing the expected classification error~\eqref{eq:discreterisk} in the infinite data regime ($n\to\infty$ and $\lambda\to 0$) \citep{vapnik2013nature}. More concretely, the minimizer $g^\star$ of the expected risk $\Expect_{(x,y)\sim\rho}S_{\M}(g(x), y)$ must satisfy $f^\star = \operatorname{sign}(g^\star)$. This property is called {\em Fisher consistency} \citep{bartlett2006convexity} (or simply consistency) and can be studied in terms of the conditional expectation~$q(x) := \rho(1|x)$, as~$f^\star(x)$ and~$g^\star(x)$ can be characterized in terms of this quantity \footnote{This is because $f^\star$ and $g^\star$ are minimizers over all measurable functions of an expectation over~$\calX$.}. Note that in the rest of the paper we will drop the dependence in~$x$ from the function~$q$: a statement~$P(q)$ for all $q\in[0,1]$ must then be read as~$P(q(x))$ for all~$x \in\calX$. Let $v_{\M}^\star(q)\subseteq\Rspace{}$ and $y^\star(q)\subseteq\calY$ be the minimizers of the conditional risks
$\Expect_{y'\sim q} S_{\M}(v, y')$ and
$\Expect_{y'\sim q}L(y, y')$, respectively. Then, Fisher consistency is equivalent to say that if
$v\in v_{\M}^\star(q)$, then $\operatorname{sign}(v)\in y^\star(q)$ for all~$q\in[0, 1]$~\citep{devroye2013probabilistic}.

\begin{figure*}[ht!]
    \centering
    \resizebox{1.018\textwidth}{!}{
    \includegraphics[width=0.3\textwidth]{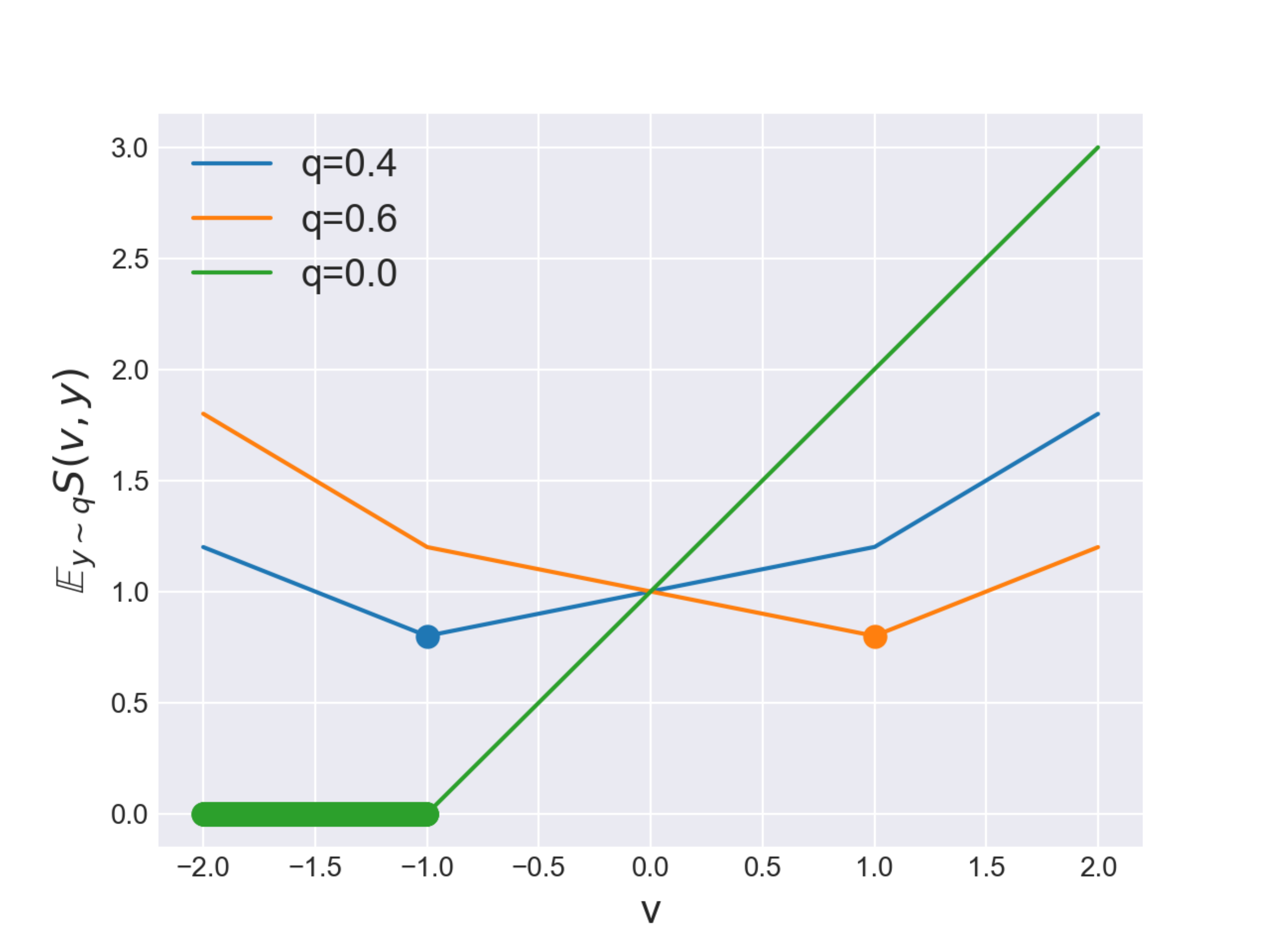}
    \includegraphics[width=0.3\textwidth]{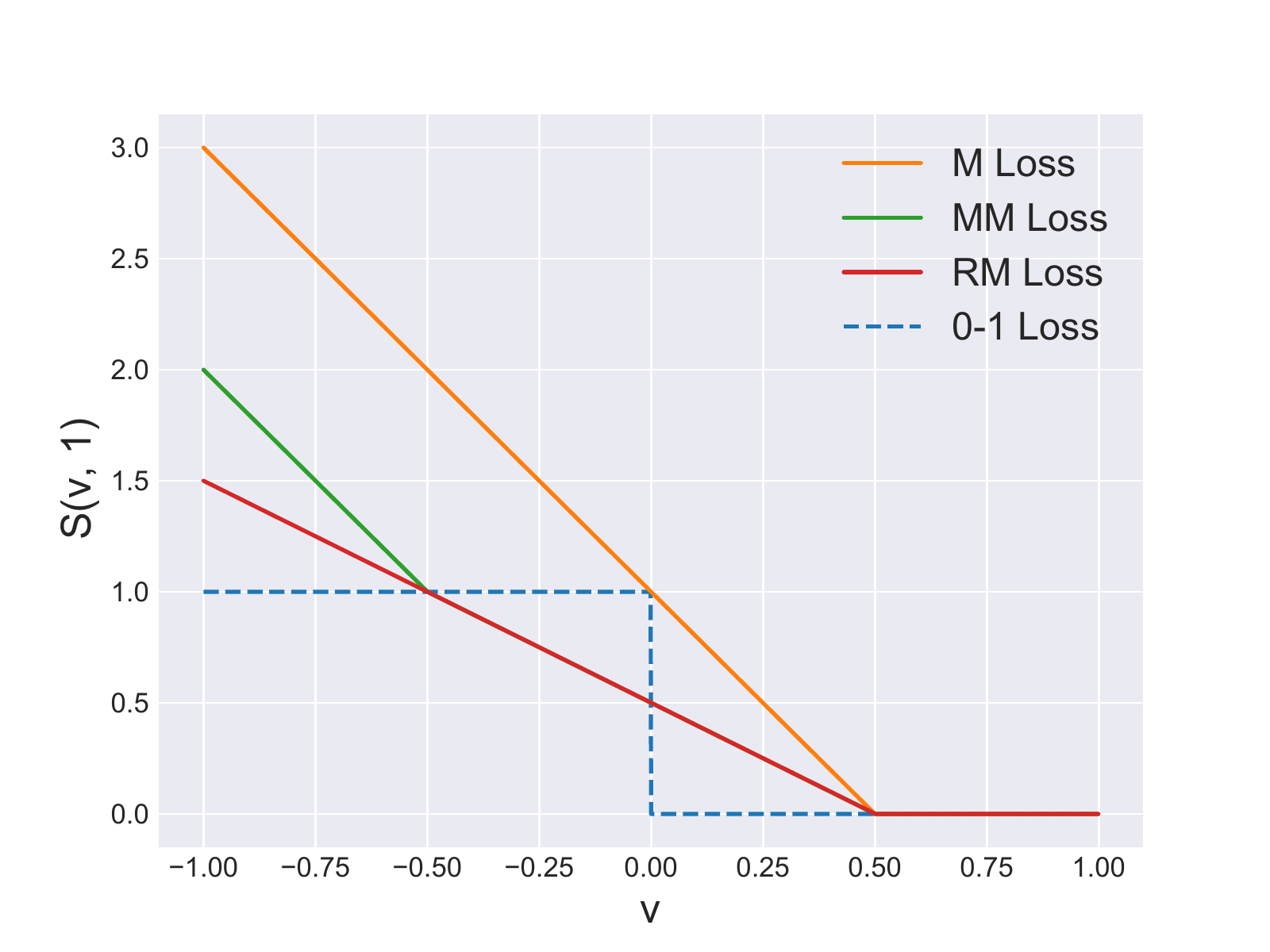}
    \includegraphics[width=0.3\textwidth]{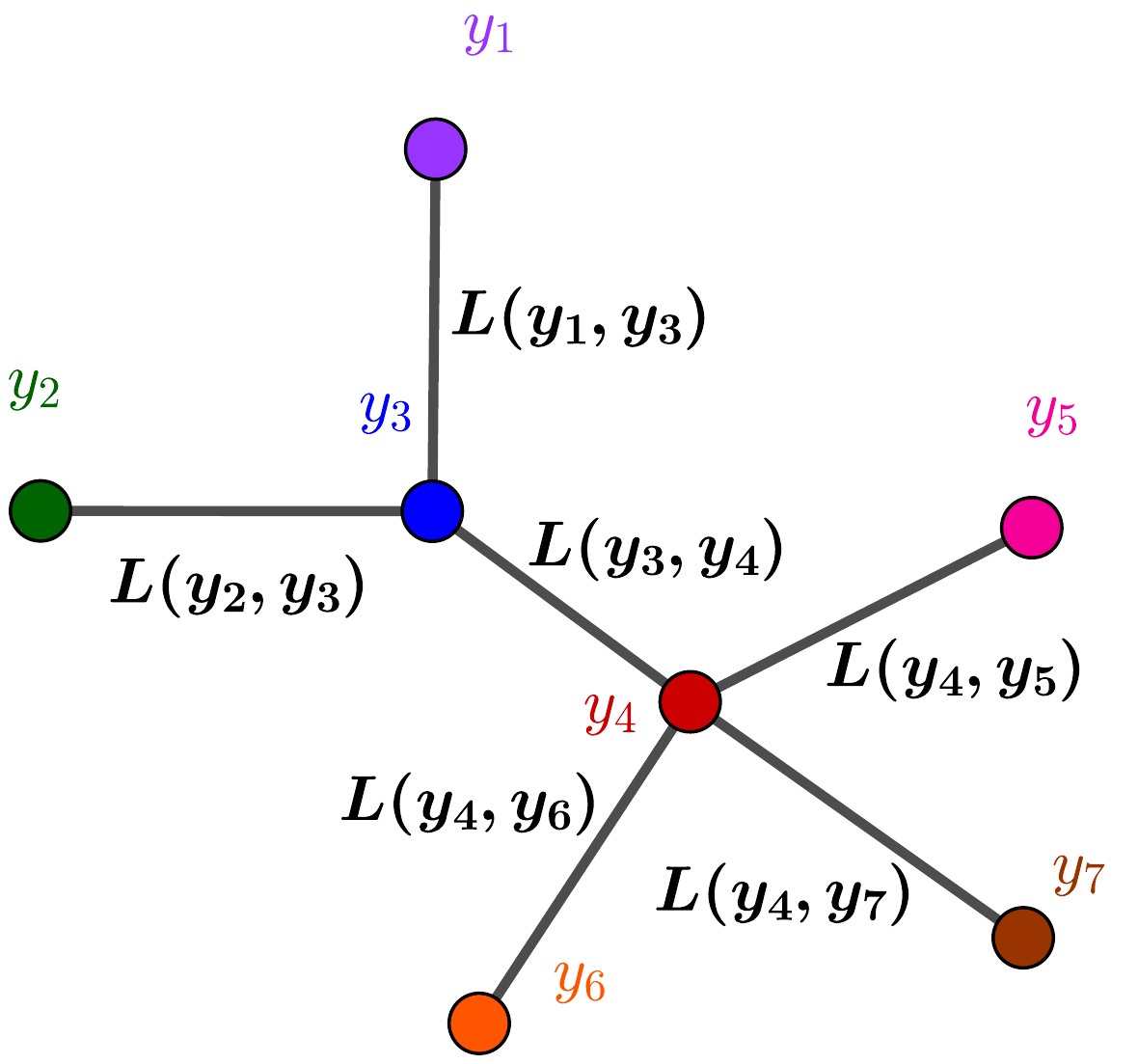}
    }
    \caption{
    \textbf{Left:} Plots of the conditional risks of the Max-Margin loss for $q=0.4, q=0.6$ and $q=0$, respectively. The conditional risk in the set of minimizers $v_{\M}^\star(q)$ is plotted with a thick point / line.
    \textbf{Middle:} Plot of $S_{\M}(v, 1), S_{\MM}(v, 1)$ and $S_{\RM}(v, 1)$ in the binary setting with $v=v_1=-v_{-1}$. In this case, $S_{\M} = 2S_{\RM}$ (so both losses generalize the binary SVM up to a factor of 2) and $S_{\M}$ is the only one upper-bounding the 0-1 loss. Moreover, the three losses are consistent with the 0-1 loss. \textbf{Right:} Distance defined in a tree: the distance/loss between two nodes is the sum of the distances between adjacent nodes of the path between them. For every triplet of outputs $y,y',y''\in\calY$, either they are aligned in a path, or there exists $z\in\calY$ belonging to the shortest path between all pairs.}
    \label{fig:condsvm}
\end{figure*}

This property is satisfied as (see also left \cref{fig:condsvm}):
\begin{align*}
    y^\star(q) &= \left\{\begin{array}{lr}
    \{1\} & q\in(1/2,1] \\
    \{-1,1\} & q = 1/2 \\
    \{-1\} & q\in[0, 1/2),
\end{array}
\right., \\
v_{\M}^\star(q) &= 
    \left\{\begin{array}{lr}
        ~[1, \infty) & q=1 \\
        ~\{1\} & q \in (1/2, 1) \\
        ~[-1, 1] & q = 1/2 \\
        ~\{-1\} & q \in (0, 1/2) \\
        ~(-\infty, -1] & q = 0. \\
    \end{array}\right. .
\end{align*}

\paragraph{Structured prediction.}
In the structured prediction setting, we have $k=|\calY|$ possible outputs and the goal is to estimate a discrete-valued function $f$ minimizing~\eqref{eq:discreterisk} where now~$L:\calY\times\calY\xrightarrow[]{}\Rspace{}$ is a generic non-negative discrete loss function between output pairs defining the task at hand. We construct predictors of the form $\argmax_{y\in\calY}~g_y(x)$, where $g:\calX\xrightarrow[]{}\Rspace{k}$ is a vector-valued function assigning scores to each of the $k$ possible outputs. The maximum margin principle from binary classification is generalized as follows. 
For every example $(x_i, y_i)$, the method minimizes the squared norm $\|g\|_{\calG}^2$ under the constraints
\begin{equation}\label{eq:lossaugmentedscores}
    g_{y_i}(x_i) \geq \underbrace{L(y, y_i) + g_{y}(x_i)}_{\text{loss-augmented scores}},
\end{equation}
for all possible outputs $y$. By writing the above constraint as $g_{y_i}(x_i) - g_{y}(x_i)\geq L(y, y_i)$, we observe that this generalizes the condition $y_ig(x_i)\geq 1$ from binary classification when~$g = g_{1} = - g_{-1}$ so that the argmax corresponds to the sign and~$L$ is the binary 0-1 loss. As in the binary case, introducing slack variables and turning it into a regularized ERM problem of the form \eqref{eq:maxmarginERM} we obtain the Max-Margin loss $S_{\M}(v, y) = \max_{y'\in\calY}~L(y, y') + v_{y'} - v_y$, 
which is constructed as a maximization of the loss-augmented scores defined in \eqref{eq:lossaugmentedscores}. To ease notation, the dependence of $S$ on the loss $L$ is deduced from the context. The Max-Margin loss \eqref{eq:maxmarginloss} is known as the Crammer-Singer SVM \citep{crammer2001algorithmic} when~$L$ is the 0-1 loss, and it is also widely used in structured prediction settings with exponentially large output spaces under the name of Structural SVM \citep{joachims2006training} or Max-Margin Markov Networks ($\operatorname{M^3N}$) \citep{taskar2004max} by using losses between structured outputs such as sequences, permutations, graphs, etc \citep{bakir2007predicting}. An interesting property of this loss is that it upper-bounds the discrete loss as $L(\argmax_{y'\in\calY}~v_{y'}, y) \leq S_{\M}(v, y)$, for all $v\in\Rspace{k}$ and $y\in\calY$, which can guide us to think that minimizing $\Expect_{(x,y)\sim \rho}S_{\M}(g(x), y)$ leads to minimizing $\Expect_{(x,y)\sim \rho}L(\argmax_{y'\in\calY}~g_{y'}(x), y)$. Unfortunately this intuition is misleading, as this bound is in general far from tight. Analogously to the binary case, let $q:\calX \to \Delta$ be the conditional distribution where~$\Delta$ is the simplex over $\calY$ (we again drop the dependence on $x$, since each statement must be read as holding for every $x \in\calX$). Moreover, we define for every $q$ the set of minimizers of the conditional risks as
\begin{align*}
    y^\star(q) &= \argmin_{y\in\calY}~L_y^\top q \subseteq\calY, \\
    v_{\M}^\star(q) &= \argmin_{v\in\Rspace{k}}~S_{\M}(v)^\top q\subseteq\Rspace{k},
\end{align*}
where $L_y = (L(y,y'))_{y'\in\calY} \in\Rspace{k}$ is the $y$-th row of the loss matrix and $S_{\M}(v) = (S_{\M}(v, y'))_{y'\in\calY}\in\Rspace{k}$. We say that~$S_{\M}$ is Fisher consistent to $L$ if for all $q\in\Delta$
\begin{equation}\label{eq:decodingconsistencymax}
     v\in v_{\M}^\star(q) \implies \underset{y\in\calY}{\argmax}~v_y\in y^\star(q). 
\end{equation}

\textbf{Related works on consistency of Max-Margin. } The Max-Margin loss is only consistent to the 0-1 loss under the dominant label assumption~$\max_{y\in\calY}q_y\geq 1/2$~\citep{liu2007fisher}. \cite{ramaswamy2018consistent}  show that it is consistent to the ``abstain'' loss, but in this case the loss appearing in the definition \eqref{eq:maxmarginloss} is not the same as the classification loss $L$. \cite{mcallester2007generalization} studies consistency of non-convex versions of max-margin methods on linear hypothesis spaces. There exist several generalizations of the binary SVM to larger output spaces other than Max-Margin \citep{dogan2016unified} such as Weston-Watkins (WW-SVM) \citep{weston1999support},  Lee-Lin-Wahba (LLW-SVM) \citep{lee2004multicategory}, Simplex-Coding (SC-SVM) \citep{mroueh2012multiclass}, with the last two being consistent and defined as sums. However, the only loss with a max-structure is~\eqref{eq:maxmarginloss}, which makes it computationally feasible to work in structured spaces of exponential size such as sequences or permutations. The Max-Min-Margin loss \citep{fathony2016adversarial, duchi2018multiclass, nowak2020consistent} (defined below in \cref{eq:maxminmargin}) is always consistent, it has a max-min structure and can be used in structured prediction settings. However, it does not correspond to the SVM in the binary setting, so it cannot be considered a generalization of the binary SVM.

\subsection{Main Results} \label{sec:mainresults}
We assume that $L$ is symmetric and that $L(y,y') = 0$ if and only if $y=y'$. Symmetry of $L$ is assumed for the sake of exposition, but it is only required for the results on Max-Margin.

\paragraph{Main Results on Max-Margin.}
The following \cref{th:mainresultmaxmargin} is our main negative result.
\begin{theorem}[Necessary condition for consistency $S_{\M}$]\label{th:mainresultmaxmargin}
Let $k=|\calY|>2$. If the Max-Margin loss is consistent to $L$, then $L$ is a distance and for every three outputs $y_1,y_2,y_3\in\calY$, there exists $z\in\calY$ for which the following three identities hold:
\begin{align*}
      L(y_1, y_2) &= L(y_1, z) + L(z, y_2), \\
    L(y_1, y_3) &= L(y_1, z) + L(z, y_3), \\
    L(y_2, y_3) &= L(y_2, z) + L(z, y_3).
\end{align*}
\end{theorem}

If $z=y_2$ in \cref{th:mainresultmaxmargin}, then the only informative condition is $L(y_1, y_3) = L(y_1, y_2) + L(y_2, y_3)$ as $L$ is assumed to be symmetric, which means that the outputs $y_1,y_2,y_3$ are `aligned' in the output space (analogously for $z=y_2,y_3$). On the other hand, if $z\neq y_1,y_2,y_3$, then the three equations are informative and all distances between the pairs can be decomposed into distances to $z$. The following discrete losses do not satisfy the above necessary condition (see Section 2 of Appendix):
\begin{itemize}
    \item[-] Losses which are not distances (such as the squared discrete loss $(y-y')^2$).
    \item[-] Losses with full rank loss matrix with existing $q\in\operatorname{int}(\Delta)$ for which all outputs are optimal, i.e.,~$y^\star(q)=\calY$ (such as the 0-1 loss).
    \item[-] Hamming losses $L(y,y')=\frac{1}{M}\sum_{m=1}^ML_m(y_m, y_m')$ with $y,y' \in\Pi_{m=1}^M\calY_m$ where $L_m$ does not satisfy the necessary conditions for some $m=1,\ldots,M$.
    \item[-] Hamming loss on permutations $L(\sigma,\sigma')=\frac{1}{M}\sum_{m=1}^M1(\sigma(m)\neq\sigma'(m))$ with $\sigma,\sigma'$ permutations of size $M$, used for graph matching \citep{petterson2009exponential, caetano2009learning}.
\end{itemize}

It is an open question whether the necessary condition of \cref{th:mainresultmaxmargin} is also sufficient. The following \cref{th:sufficientcondition} shows that distances defined in a tree, which always satisfy this condition (see right \cref{fig:condsvm}), are indeed consistent. 

\begin{theorem}[Sufficient condition for consistency $S_{\M}$]\label{th:sufficientcondition}
If $L$ is a distance defined in a tree, then the Max-Margin loss is consistent to $L$.
\end{theorem}

An important example of these losses is the \emph{absolute deviation} loss used in ordinal regression,
\begin{equation*}
L(y,y') = |\gamma_{y} - \gamma_{y'}|, \hspace{0.5cm}\gamma\in\Rspace{k},    
\end{equation*}

\begin{figure*}[ht!]
    \centering
    \hspace{-0.01\textwidth}
    \resizebox{1.018\textwidth}{!}{
    \includegraphics[width=0.04\textwidth]{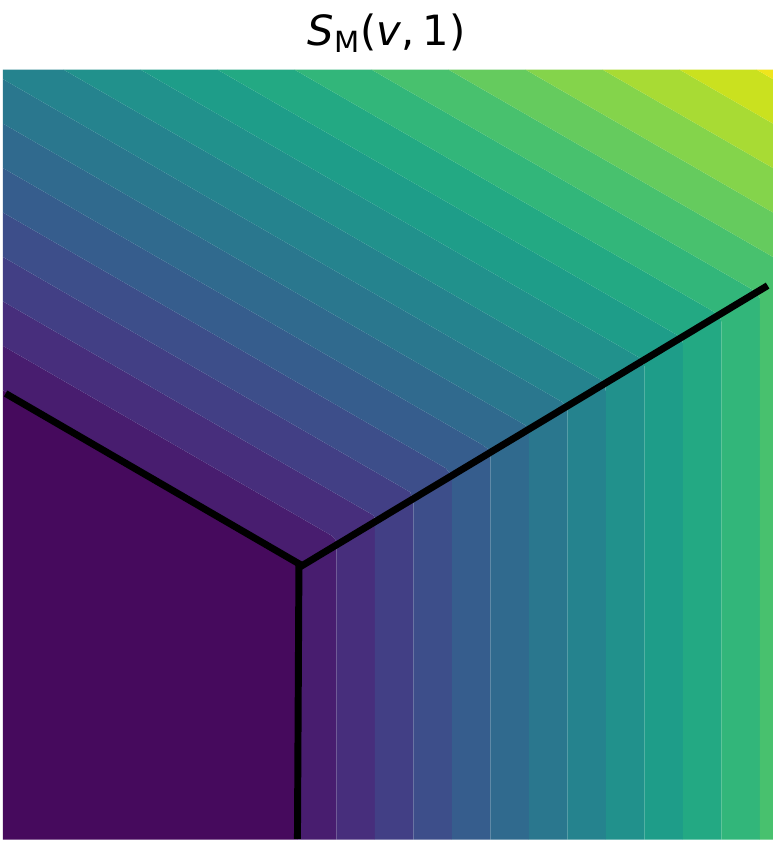}
    \includegraphics[width=0.04\textwidth]{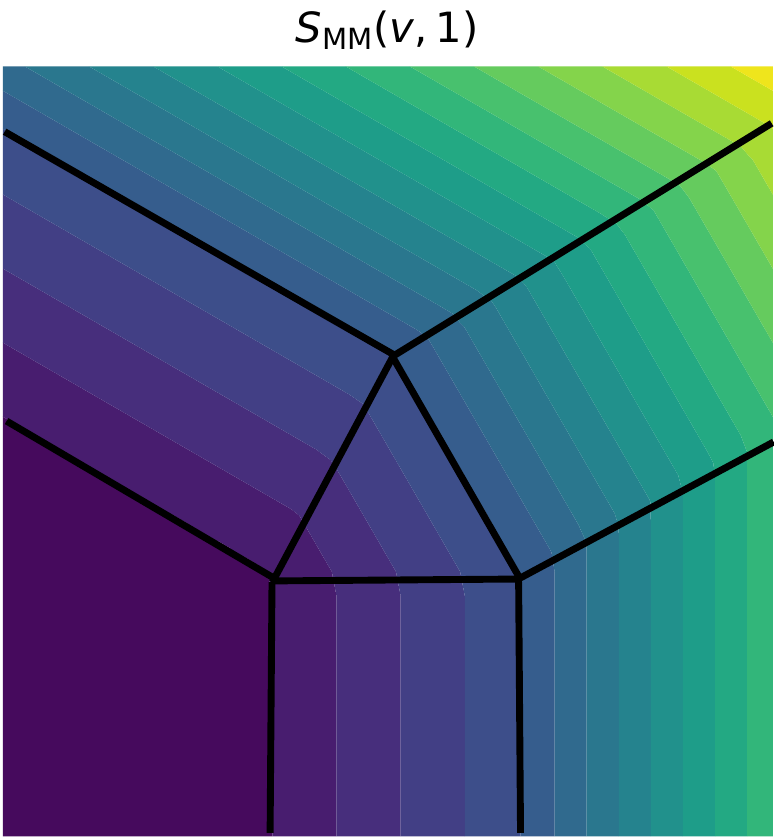}
    \includegraphics[width=0.04\textwidth]{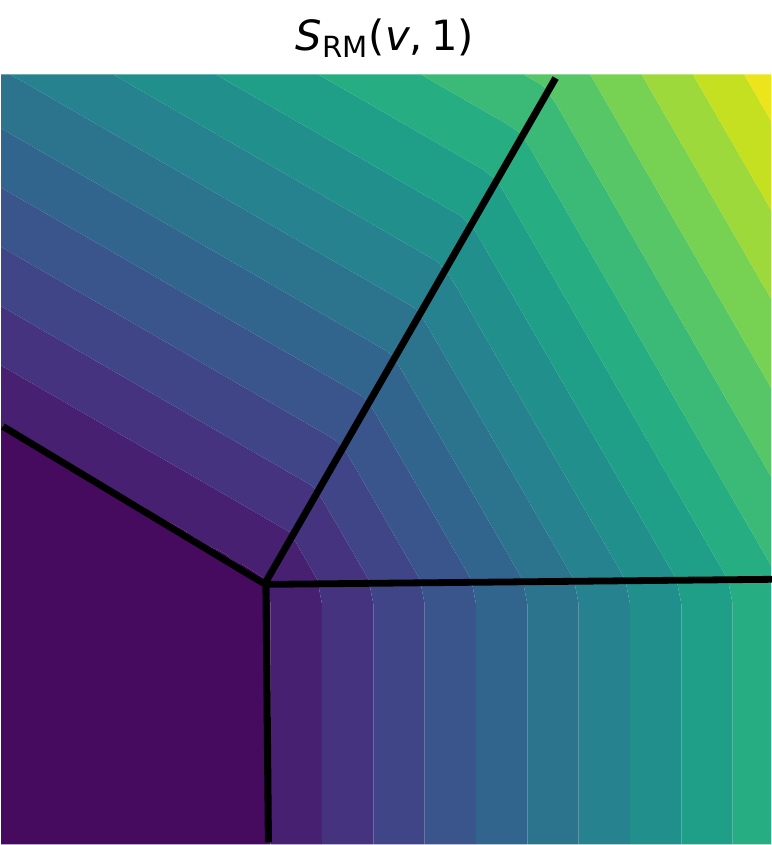}
    }
    \vspace*{-.25cm}
    \caption{From left to right: plots of $S_{\M}(v, 1), S_{\MM}(v, 1)$ and $S_{\RM}(v, 1)$ in the three-label setting with $v^\top 1 = 0$ for the 0-1 loss (note that $S(v+c1,y)=S(v,y)$ for all three losses). The Max-Min-Margin and the Restrictive-Max-Margin loss coincide in the bottom-left region, but the Max-Min-Margin loss has three extra activated faces in the top-right region of the plot which are in general unnecessary for consistency, while the Restricted-Max-Margin uses just the necessary ones. The Max-Margin in the left plot uses just three faces, being insufficient for consistency.}
    \label{fig:geometries}
\end{figure*}

for which the associated tree is a chain. Note that these losses are not the only ones satisfying the necessary condition given by \cref{th:mainresultmaxmargin}. Indeed, the Hamming loss with $M=2$, $\calY_1=\calY_2=\{-1,1\}$ and $L_1,L_2$ the 0-1 loss is \emph{not} a distance in a tree, satisfies the necessary condition, and consistency can be proven to hold (see Section 2 of Appendix). The following \cref{prop:partialconsistency} gives a much milder sufficient condition to ensure \emph{partial} consistency under the dominant label assumption, thus generalizing the well-known results from \cite{liu2007fisher}.

\begin{proposition}[Sufficient condition for \emph{partial} consistency $S_{\M}$]\label{prop:partialconsistency} If $L$ is a distance, then the Max-Margin loss is consistent to $L$ under the dominant label assumption, i.e., $\max_{y\in\calY}~q_y\geq 1/2$.
\end{proposition}

In other words, if the learning task is defined by a distance and it is close to deterministic, then the Max-Margin loss is consistent to the task.

\paragraph{Beyond Max-Margin.}
To overcome the limitations imposed by the maximum margin, but retaining the maximization structure of the loss, we propose a novel generalization of the binary SVM to structured prediction by restricting the maximization of the loss-augmented scores in~\eqref{eq:maxmarginloss}. First, note that the Max-Margin loss can be written as a maximization over the simplex ${\color{red} \Delta}$ over $\calY$ as 
\begin{equation}\label{eq:mm}
    S_{\M}(v, y) = \max_{q\in{\color{red}\Delta}}~L_y^\top q + v^\top q - v_y.
\end{equation}
We restrict the maximization to the so-called \emph{prediction set}~${\color{blue}\Delta(y)} = \{q\in\Delta~|~y\in y^\star(q)\}$, defined as the set of probabilities for which $y$ is optimal. In binary classification the sets are $\Delta(-1)=[0, 1/2]$ and $\Delta(1)=[1/2,1]$. The resulting \emph{Restricted-Max-Margin} loss reads
\begin{equation}\label{eq:rmm}
    S_{\RM}(v, y) = \max_{q\in{\color{blue}\Delta(y)}}~L_y^\top q + v^\top q - v_y.
\end{equation}
This loss satisfies $2S_{\RM} = S_{\M}$ in the binary setting (see middle \cref{fig:condsvm}), thus, it corresponds to the binary SVM up to a scaling with a factor of two. The following \cref{th:rmmconsistency} states that consistency of $S_{\M}$ implies consistency of $S_{\RM}$ and provides a sufficient condition for consistency of~$S_{\RM}$.

\begin{theorem}[Sufficient condition for consistency $S_{\RM}$]\label{th:rmmconsistency} The Restricted-Max-Margin loss is consistent to $L$ whenever the Max-Margin is consistent. Moreover, if $L$ satisfies $q_y > 0$ for every~$y$ optimal for $q\in\Delta$, i.e., $q\in\Delta(y)$, then the Restricted-Max-Margin is also consistent to $L$.
\end{theorem}

In other words, if the output $y$ is optimal for $q$, then the probability of this label has to be strictly greater than zero $q_y>0$. The 0-1 loss, which does not satisfy the necessary condition of \cref{th:mainresultmaxmargin}, satisfies the sufficient condition for the Restricted-Max-Margin, as $\min_{q\in\Delta(y)}q_y = 1/k$ for all $y\in\calY$. However, there are still losses for which \eqref{eq:rmm} is not consistent to, such as the squared discrete loss~$(z-y)^2$ (see right of \cref{fig:predictionsets}). The remaining inconsistencies can be resolved by going beyond the maximization structure into a max-min structure. The resulting loss is the so-called \emph{Max-Min-Margin loss} \citep{fathony2016adversarial, duchi2018multiclass, nowak2020consistent} defined as
\begin{equation}\label{eq:maxminmargin}
    S_{\MM}(v, y) = \max_{q\in\Delta}\boldsymbol{{\color{ForestGreen}\min_{z\in\calY}}}~L_{\boldsymbol{\color{ForestGreen}z}}^\top q + v^\top q - v_y.
\end{equation}
It is known \citep{nowak2020consistent} that the loss \eqref{eq:maxminmargin} is \emph{always} consistent to $L$. As shown in \cref{fig:condsvm} (middle), this loss does not correspond to the SVM in the binary setting because it has two symmetric kinks instead of one. See also \cref{fig:geometries} to compare the shape of the different losses for $k=3$. Hence, while the structure of the loss gets more computationally involved from the Max-Margin loss \eqref{eq:mm} to the Max-Min-Margin loss \eqref{eq:maxminmargin}, passing by the Restricted-Max-Margin loss \eqref{eq:rmm}, the consistency properties of these losses improve from one to the next.

\section{BACKGROUND AND PRELIMINARY RESULTS}\label{sec:3}
\subsection{Background on Polyhedral Losses}
\paragraph{Fisher consistency.}
Let us now consider a generic loss $S:\Rspace{k}\times\calY\rightarrow\Rspace{}$ and let's generalize the argmax computing the prediction from the scores in the previous section to a generic \emph{decoding} function~$d:\Rspace{k}\rightarrow\calY$. The set
of minimizers~$v^\star(q)\subseteq\Rspace{k}$ of the conditional risk $S(v)^\top q$ is also defined as before. We say that $S$ is \emph{Fisher consistent} to $L$ \citep{tewari2007consistency} under the decoding~$d:\Rspace{k}\rightarrow\calY$ if
\begin{equation}\label{eq:decodingconsistency}
    v\in v^\star(q) \implies d(v)\in y^\star(q), 
\end{equation}
for all $q\in\Delta$. If the decoding is not specified, it means that there exists a decoding satisfying this property.
An important quantity throughout the paper is the \emph{Bayes risk}, a \emph{concave} function defined as the minimum of the conditional expected loss respectively for $L$ and $S$:
\begin{equation*}
    H_L(q) = \min_{y\in\calY}~L_y^\top q, \hspace{0.8cm} H_S(q) = \min_{v\in\Rspace{k}}~S(v)^\top q.
\end{equation*}

\paragraph{Embedding of discrete losses.} Fisher consistency in~\cref{eq:decodingconsistency} states that every minimizer~$v\in v^\star(q)$ can be assigned to a solution of the discrete task using the decoding $d$. For the analysis of this paper, it will be useful to work using the concept of embeddability between losses \citep{finocchiaro2019embedding}, a stronger notion than Fisher consistency.

\begin{definition}[Embeddability]\label{def:embeddability} $S$ \emph{embeds} $L$ if there exists an embedding $\psi:\calY\rightarrow\Rspace{k}$
such that: (i) $y\in y^\star(q) \iff \psi(y) \in v^\star(q), ~\forall q\in\Delta$, and (ii) $S(\psi(y)) = L_{y}, ~\forall y\in\calY$.
\end{definition}

Condition (i) states that every solution of the discrete problem corresponds to a solution of the problem in $S$ and vice versa. In particular, this rules out many smooth plug-in classifiers such as the squared loss or logistic regression, because they predict the vector of probabilities~$q$ which cannot be recovered from the discrete predictor~$y^\star(q)$. 
It is known \citep{finocchiaro2019embedding} that the existence of an embedding~$\psi$ satisfying (i) implies the existence of a decoding~$d$ satisfying \cref{eq:decodingconsistency}, so it is already a sufficient condition for Fisher consistency. Note that both Eq.~\eqref{eq:decodingconsistency} and condition (i) are assumptions on the predictors~$y^\star$ and~$v^\star$, but there exist several losses~$L$ with the same set of minimizers~$y^\star(q)$ of the conditional risk~$L_y^\top q$. The same can be said for the objects~$v^\star$ and~$S$. Condition (ii) restricts the relationship between pairs of losses by assuming that~$L$ can be recovered from~$S$ using the embedding~$\psi$. The following \cref{prop:bayesrisksequal} shows that~$S$ embedding~$L$ is equivalent to having the same Bayes risks.
\begin{proposition}[\cite{finocchiaro2019embedding}]\label{prop:bayesrisksequal}
$S$ embeds $L$ if and only if $H_L = H_S$.
\end{proposition}
Moreover, it is known that any discrete loss $L$ is embedded by at least one loss $S$ (Theorem 2 in \cite{finocchiaro2019embedding}), which corresponds precisely to the  Max-Min-Margin loss $S_{\MM}$ defined in \cref{eq:maxminmargin}. Indeed, $S_{\MM}$ and~$L$ have the same Bayes risk as
\begin{align*}
    H_{\MM}(q) &= \min_{v\in\Rspace{k}}\big(\max_{p\in\Delta}\min_{y\in\calY}~L_{y}^\top p + v^\top p\big) - v^\top q \\
    &= \min_{v\in\Rspace{k}}(-H_{L})^*(v) - v^\top q = H_L(q),
\end{align*}
where $h^*(u) = \sup_{s\in\Rspace{k}}~u^\top s - h(s)$ is the Fenchel conjugate of $h$. It can be checked \citep{nowak2020consistent} that the embedding is~$\psi(y) = -L_y$ and it is always Fisher consistent to $L$ under the argmax decoding. 

\subsection{Preliminary Results}
\paragraph{Relationship between losses and Bayes risks.} What makes the Max-Min-Margin loss simple to analyze is its Fenchel-Young structure \citep{blondel2019learning}, i.e., it can be written in the form $S(v, y) = \Omega^*(v) - v_y$, for a certain convex function $\Omega$ defined in the simplex. We extend this notion by allowing the convex function~$\Omega$ to depend on the label $y$ as
\begin{equation}\label{eq:FYform}
    S(v, y) = (\Omega^y)^*(v) - v_y.
\end{equation}
The losses~$S_{\M}, S_{\RM}$ and $S_{\MM}$ can be written in this form with
\begin{align*}
    \Omega_{\MM}^y(q) &= -\min_{y'\in\calY}L_{y'}^\top q + i_{\Delta}(q), \\
    \Omega_{\M}^{y}(q) &= -L_y^\top q + i_{\Delta}(q), \\
    \Omega_{\RM}^{y}(q) &= -L_y^\top q + i_{\Delta(y)}(q),
\end{align*}
where $i_{U}(u) = 0$ if $u\in U$ and $\infty$ otherwise. The first equation is the only one independent of $y$ and we remove its dependence by simply writing $\Omega_{\MM}$. The following \cref{prop:lossrelations} relates the functions $\Omega_{\MM}, \Omega_{\M}^y$ and $\Omega_{\RM}^y$.

\begin{proposition}\label{prop:lossrelations}
The following holds: $\Omega_{\MM} = \max_{y\in\calY}~\Omega_{\M}^{y}$ and $(\Omega_{\MM})^* = \max_{y\in\calY}~(\Omega_{\RM}^{y})^*$.
\end{proposition}
\begin{proof}
The first identity is trivial from the definition. For the second identity, note that by construction the prediction sets necessarily cover the simplex as~$\Delta = \cup_{y'\in\calY}~\Delta(y')$. Hence,
\begin{align*}
    (\Omega_{\MM})^*(v) 
    &= \max_{y\in\calY}\big(\max_{q\in\Delta(y)}\min_{y'\in\calY}~L_{y'}^\top q + v^\top q\big) \\ &=\max_{y\in\calY}\big(\max_{q\in\Delta(y)}~L_{y}^\top q + v^\top q\big) \\
    &=  \max_{y\in\calY}~(\Omega_{\RM}^{y})^*(v),
\end{align*}
where we have used the fact that $\min_{y'\in\calY}L_{y'}^\top q=L_y^\top q$ whenever $q\in\Delta(y)$ by construction.
\end{proof}
Moreover, it can be readily seen that $S_{\RM}(v, y) \leq S_{\MM}(v, y) \leq S_{\M}(v, y)$, 
for all $v\in\Rspace{k}$ and $y\in\calY$. See \cref{fig:condsvm} (left and middle) and \cref{fig:geometries} for the shape of these losses when $L$ is the 0-1 loss for two and three dimensions, respectively. 
Our main quantity of interest is the Bayes risk~$H_S$, because by comparing it with~$H_L$ we are able to tell whether~$L$ is embedded by $S$ using \cref{prop:bayesrisksequal}, thus proving consistency. The following \cref{prop:bayesrisks} gives the form of the Bayes risks.
\begin{proposition}[Bayes risks]\label{prop:bayesrisks}
For all $q\in\Delta$, the Bayes risks read
\begin{align*}
    H_{\MM}(q) &= H_L(q), \\
    H_{\M}(q) &= \underset{Q\in U(q, q)}{\max}\langle L, Q\rangle_{\F}, \\
    H_{\RM}(q) &= \underset{Q\in U(q, q)\cap \calC_L}{\max}\langle L, Q\rangle_{\F},
\end{align*}
where $U(q, q) = \{Q\in\Rspace{k\times k}~|~Q1=q, Q^\top 1=q, Q\succeq 0\}$, ~$\calC_L = \{Q\in\Rspace{k\times k}~|~(1L_y^\top - L)Q_y\preceq 0,~\forall y\in\calY\}$ and~$\langle \cdot,\cdot\rangle_{\F}$ denotes the Frobenius inner product. Moreover, we have that $H_{\RM} \leq H_{\MM} \leq H_{\M}$, 
and there exists~$L$ for which~$H_L \neq H_{\M}$ and/or $H_{\RM} \neq H_{L}$.
\end{proposition}

The proof can be found in Section 1 of Appendix. As a corollary, the only one of these losses always embedding~$L$ is the Max-Min-Margin loss. Our sufficient conditions for consistency will correspond to the conditions on $L$ for which the Bayes risks $H_{\M}$ and/or $H_{\RM}$ are equal (or proportional) to $H_L$, thus implying consistency.
In the next section, we use the expressions of the Bayes risks given by \cref{prop:bayesrisks} as a basis to prove the consistency results.

\section{FISHER CONSISTENCY ANALYSIS}\label{sec:4}
\subsection{Analysis of Max-Margin Loss}
In this section we want to provide a necessary condition for consistency of the Max-Margin loss (\cref{th:mainresultmaxmargin}). Note that it is not enough to provide a condition for which $H_S\neq H_L$, as $S$ consistent to $L$ does not imply $S$ embedding $L$. The following \cref{prop:necessaryextremes} gives a necessary condition for consistency in terms of the extreme points of the prediction sets and the Bayes risk of $S$.

\begin{proposition}\label{prop:necessaryextremes}
Assume $S$ is Fisher consistent to $L$. Then, for any extreme point $q\in\Delta$ of a prediction set $\Delta(y),~y\in\calY$, we necessarily have $\{q\} = \partial(-H_{S})^*(v)$ for some $v\in\Rspace{k}$.
\end{proposition}

The set $\partial h(x)$ denotes the subgradient of the function $h$.
This result is proven in Section 2 of Appendix. To use this necessary condition, we first have to compute the Fenchel conjugate $(-H_{\M})^*$ of the Max-Margin loss. This is given by the following \cref{prop:computationsCS}.

\begin{proposition}\label{prop:computationsCS}
If $L$ is symmetric, then 
\begin{equation*}
    (-H_{\M})^*(v) = \max_{y,y'\in\calY}~L(y,y') + \frac{v_{y} + v_{y'}}{2},
\end{equation*}
for all $v\in\Rspace{k}$.
\end{proposition}

As a corollary, we obtain the specific form of the images of the sub-gradient mapping~$\partial(-H_{\M})^*$ at the differentiable points, i.e., when the sub-gradient is a singleton.

\begin{corollary}\label{cor:imagesubgradients}
The $0$-dimensional images of $\partial(-H_{\M})^*$ are of the form $q = \frac{1}{2}(e_y + e_{y'}), y,y'\in\calY$.
\end{corollary}

\begin{figure*}[ht!]
    \centering
    \resizebox{1.018\textwidth}{!}{
    \includegraphics[width=0.25\textwidth]{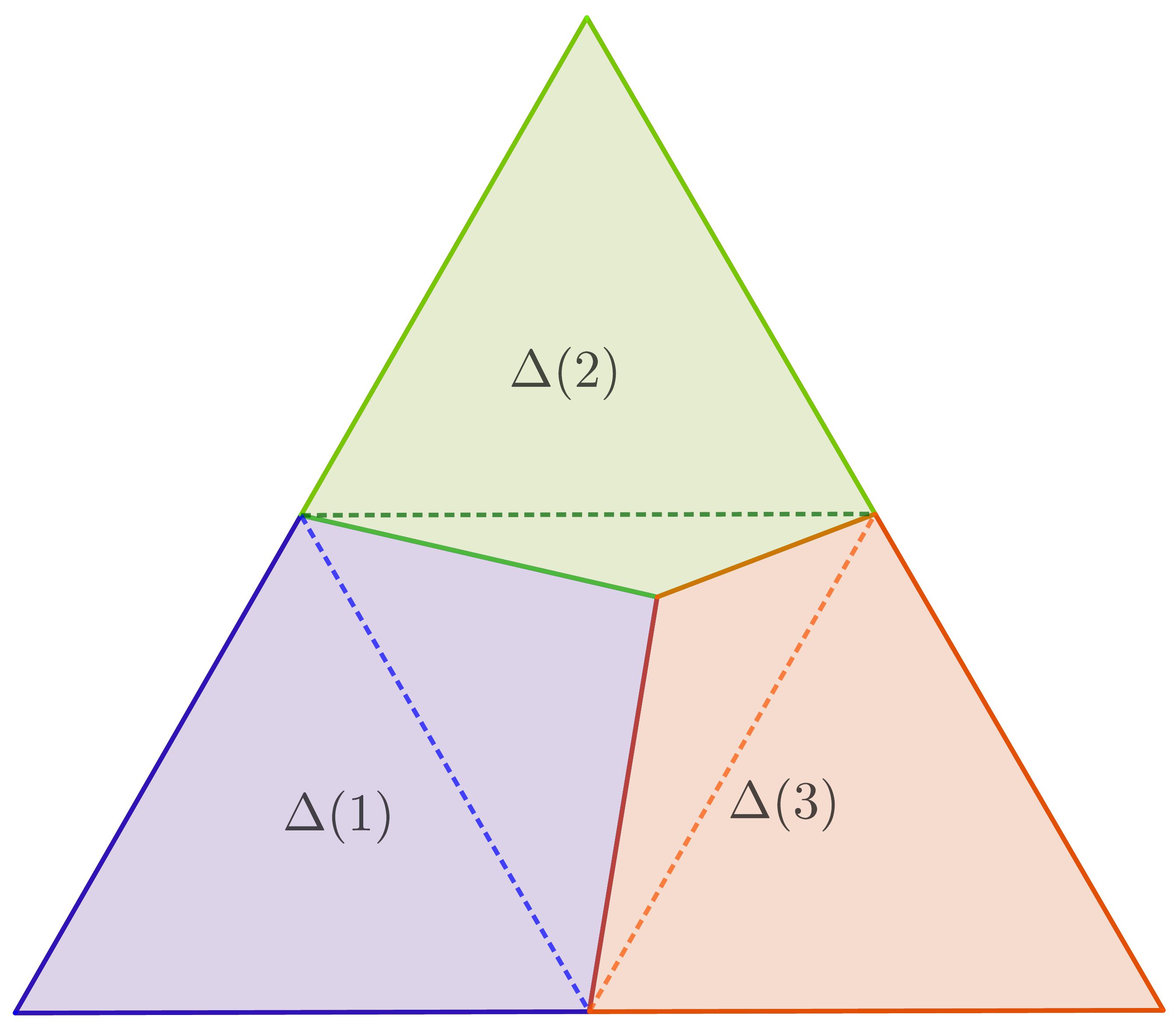} \hspace{.6cm}
    \includegraphics[width=0.25\textwidth]{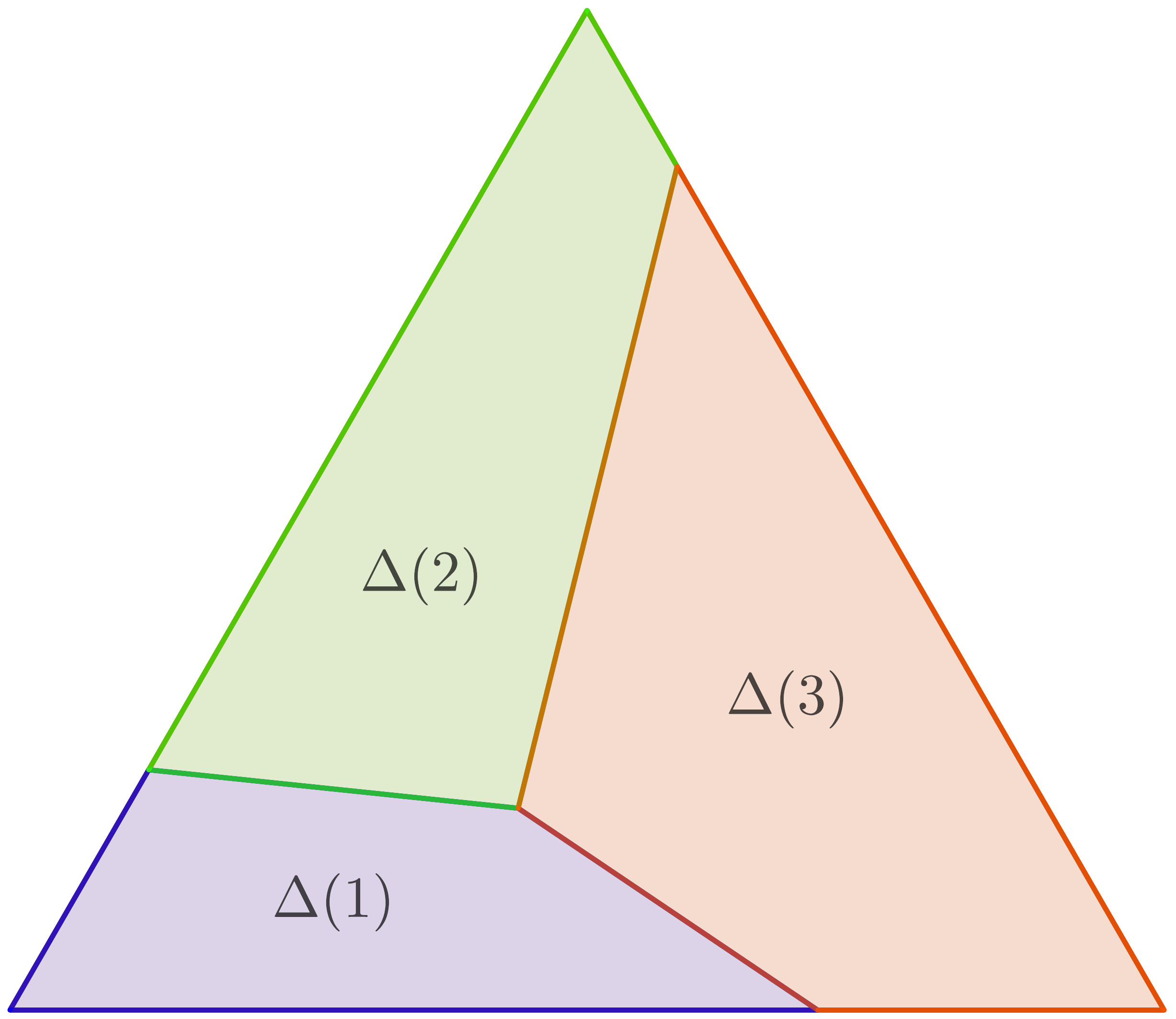}
    \hspace{.6cm}
    \includegraphics[width=0.25\textwidth]{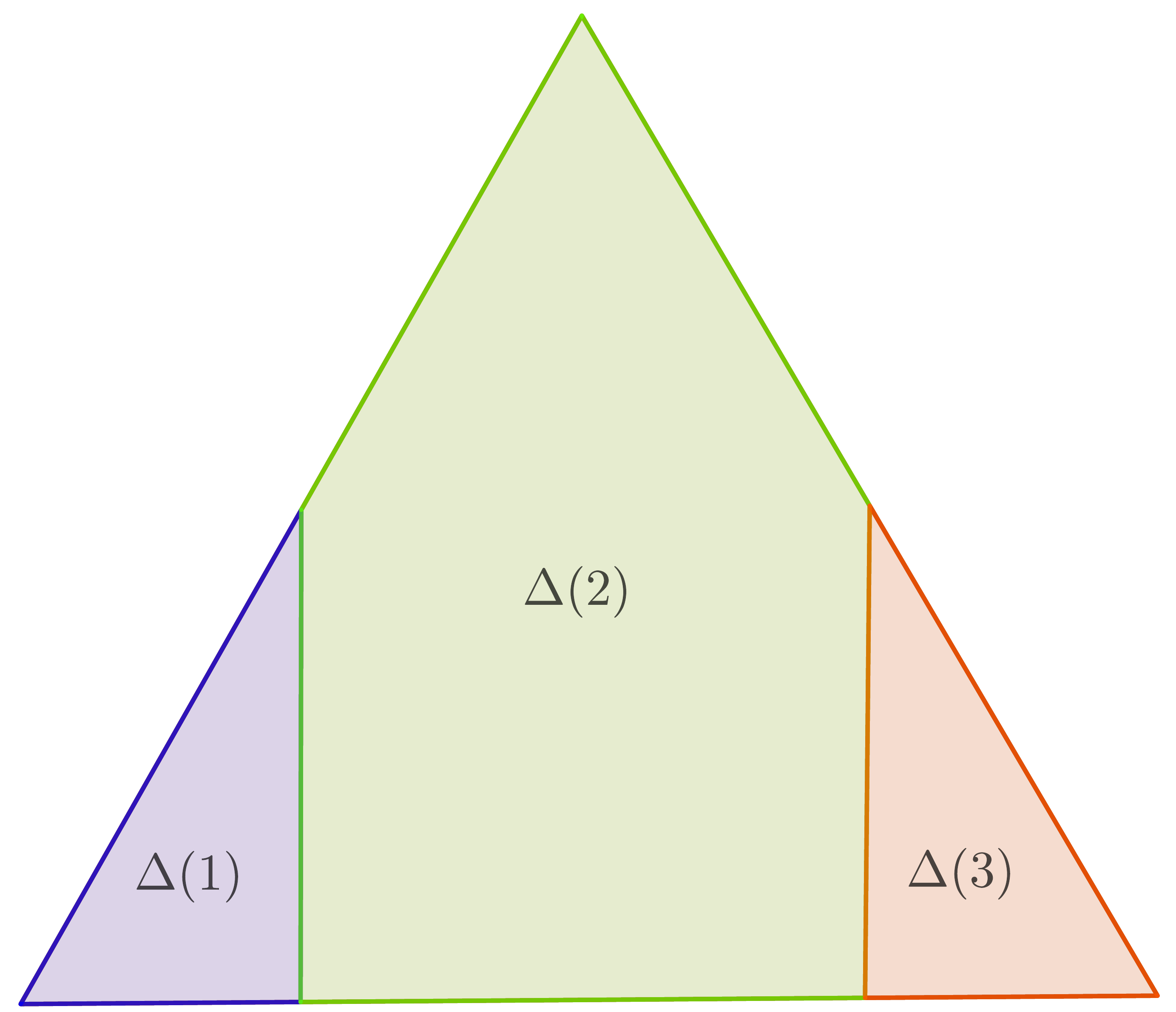}
    }
    \caption{\textbf{Left:} Prediction sets $\Delta(y)$ for a symmetric loss. The points given by \cref{cor:imagesubgradients} are the 3 extremes of the simplex and the middle points in the faces of the simplex. The interior point is not of the form $\frac{1}{2}(e_y + e_{y'})$ for any $y,y'\in\{1, 2, 3\}$. \textbf{Middle:} The sufficient condition from \cref{prop:extremepointsrm} is satisfied for any point in the simplex. \textbf{Right:} The sufficient condition from \cref{prop:extremepointsrm} is not satisfied as the prediction set $\Delta(2)$ (in green) intersects the line~$q_2 = 0$. Moreover, inconsistency can be shown in this case by comparing the affine regions of the Bayes risk $H_{\RM}$ with these predictions sets.}
    \label{fig:predictionsets}
\end{figure*}

Note that the points $\{q\} \in \operatorname{Im}\partial(-H_{\M})^*$ are independent of the discrete loss $L$. In \cref{fig:predictionsets} (left), we show that this is not the case for the extreme points of the prediction sets $\Delta(y)$'s of a generic~$L$. In this example, the points $\frac{1}{2}(e_{y} + e_{y'})$'s are extreme points of the prediction sets because $L$ is symmetric, but the extreme point in the interior is not of this form. By moving this point to $\frac{1}{2}(e_{1} + e_{3})$, we enforce $L(1, 2) + L(2, 3) = L(1, 3)$, which corresponds precisely to the necessary condition given by \cref{th:mainresultmaxmargin} when $z=y_2$. The generic necessary condition is obtained by extending this argument to consider all possibilities in larger output spaces. The full proof can be found in Section 2 of Appendix.

\paragraph{Consistency of distances defined in trees.} This result is proved by showing ~$H_{\M} = 2H_L$ (~$H_S\propto H_L$ also implies consistency \citep{finocchiaro2019embedding}) whenever $L$ is a distance defined in a tree, and it is done by proving equality of the Fenchel conjugates $(-H_{\M})^* = (-2H_L)^*$. The proof is based on the analysis of the extreme points of a certain polytope defined in terms of $L$, and can be found in Section 2 of Appendix.

\paragraph{Partial positive results on consistency.} In this section, we generalize the well-known result by~\cite{liu2007fisher} which states that the Max-Margin loss is consistent to the 0-1 loss under the dominant label condition. More specifically, we provide a generalization of this result to all distance losses, i.e., symmetric and satisfying the triangle inequality.

\begin{proposition}\label{prop:partialconsistency2} If $L$ is a distance, then we have $H_{\M} \leq 2H_{L}$ and $H_{\M}(q) = 2H_L(q)$ ($S$ embeds $2L$, thus consistent), for all $q\in\Delta$ such that $\max_{y\in\calY}q_y \geq 1/2$.
\end{proposition}

Some intuition can be obtained for $k=3$. The interior of the set delimited by the dashed lines in the left image from~\cref{fig:predictionsets} shows precisely the points where~$\max_{y\in\calY}~q_y\leq 1/2$. If $L$ is a distance, then the interior extreme point can only move inside this region, and hence, it remains consistent at the exterior of this set, which is precisely where the dominant label condition is satisfied. The proof of this result can be found in~Section 2 of Appendix.

\subsection{Analysis of Restricted-Max-Margin Loss}

In this section we sketch the proof of the consistency of the Restricted-Max-Margin loss (\cref{th:rmmconsistency}). We prove that~$S_{\RM}$ embeds~$L$ by showing equal Bayes risks using the expressions from~\cref{prop:bayesrisks}.
The proof is done in two steps. In the first one (\cref{prop:rmfirststep}), we show that under a condition on the extreme points of the prediction sets both Bayes risks are equal, while in the second we obtain that this condition is satisfied whenever the hypothesis from \cref{th:rmmconsistency} holds.
\begin{proposition}\label{prop:rmfirststep} If $qq^\top \in\calC_L = \{Q\in\Rspace{k\times k}~|~(1L_y^\top - L)Q_y\preceq 0,~\forall y\in\calY\}$ for all
extreme points $q\in\Delta$ of the prediction sets $\Delta(y)$'s, then $H_{\RM} = H_L$.
\end{proposition}
\begin{proof} By \cref{prop:bayesrisks}, we already know that $H_{\RM}\leq H_L$, so we only need to show that $H_{\RM}\geq H_L$. If $q$ is an extreme point of the prediction set $\Delta(y)$, then $H_{\RM}(q)\geq H_L(q)$ by taking the matrix~$qq^\top \in U(q, q)\cap\calC_L$, which satisfies $\langle L, qq^\top\rangle_{\F} = \sum_{y'}q_{y'}L_{y'}^\top q \geq \sum_{y'}q_{y'}L_{y}^\top q = L_y^\top q = H_L(q)$,
where we have used that $L_{y'}^\top q \geq L_y^\top q$ for all $y'\in\calY$ as~$q\in\Delta(y)$. If $q$ is not an extreme point, then it can be written as a convex combination of extreme points~$q_i$ of~$\Delta(y)$ as~$q = \sum_{i=1}^m\alpha_iq_i$ with~$\alpha^\top 1 = 1, \alpha\succeq 0$. Then, it is straightforward to show that the matrix $Q = \sum_{i=1}^m\alpha_iq_iq_i^\top$ is in~$U(q,q)\cap\calC_L$, and satisfies $\langle L, Q \rangle_{\F}\geq L_y^\top q = H_L(q)$. More details in Section 3 of Appendix.
\end{proof}

\begin{proposition}\label{prop:extremepointsrm}
Assume that for every $q\in\Delta$, if $y$ is optimal for $q$ (i.e., $q\in\Delta(y)$), then~$q_y > 0$.  In this case, all extreme points of the prediction sets satisfy $qq^\top\in\calC_L$.
\end{proposition}

The proof of this result and the fact that consistency of $S_{\M}$ implies consistency of $S_{\RM}$ can be found in Section 3 of Appendix. In \cref{fig:predictionsets} (middle and right), we show geometrically in dimension $k=3$ what this condition means. 

\subsection{Argmax Decoding}

The positive consistency results of both $S_{\M}$ and $S_{\RM}$ do not specify whether the argmax decoding~$d(v) = \argmax_{y\in\calY}~v_y$ is consistent, but just that there exists a decoding for which consistency holds. From \cite{nowak2020consistent}, we know that the set of minimizers of the Max-Min-Margin loss $S_{\MM}$ is
\begin{equation}\label{eq:minimizersmm}
    v_{\MM}^\star(q) = -\operatorname{hull}(L_y)_{y\in y^\star(q)} + \calN_{\Delta}(q),
\end{equation}
for all $q\in\Delta$, where $\operatorname{hull}$ denotes the convex hull and $\calN_{\Delta}(q) = \{u\in\Rspace{k}~|~u^\top(p - q)\leq 0, \forall p\in\Delta\}$ is the normal cone of the simplex at the point~$q$. In this case, the argmax decoding is consistent because~$\argmax_{y}v_y\in y^\star(q)$ whenever $v\in v_{\MM}^\star(q)$ if \cref{eq:minimizersmm} is satisfied. The following \cref{prop:argmax} shows that the same holds true for the other losses whenever they embed $L$ (or $2L$).

\begin{theorem}\label{prop:argmax}
$S_{\M}$ and $S_{\RM}$ are consistent to $L$ under the argmax decoding whenever $H_{\M}=2H_L$ ($S_{\M}$ embeds $2L$) and $H_{\RM}=H_L$ ($S_{\RM}$ embeds $L$), respectively.
\end{theorem}

\section{CONCLUSION AND FUTURE DIRECTIONS}\label{sec:6}

In this work, we have analyzed the consistency properties of the well-known Max-Margin loss for general classification tasks. We show that a restrictive condition on the task, which is not satisfied by most of the used losses in practice, is necessary for Fisher consistency. We also show consistency for the absolute deviation loss used in ordinal regression, and more generally to any distance defined in a tree, showing that the necessary condition is meaningful. Moreover, we have overcome this limitation and introduced a novel generalization of the binary SVM loss called Restricted-Max-Margin loss, which maintains the maximization over the loss-augmented scores and it is consistent under milder conditions on the task at hand. Two important questions remain unanswered: first, whether the proposed necessary condition for the Max-Margin is sufficient or not, and secondly, whether there exist tractable linear maximization oracles over the prediction sets $\Delta(y)$ for specific structured prediction problems, which would make the Restricted-Max-Margin loss more attractive than the Max-Min-Margin loss in practice.

\subsubsection*{Acknowledgements}
This work was funded in part by the French government under management of Agence Nationale de la Recherche as part of the ‘Investissements d’avenir’ program, reference ANR-19-P3IA-0001 (PRAIRIE 3IA Institute). We acknowledge support of the European Research Council (grant SEQUOIA 724063). We also acknowledge support of the European Research Council (grant REAL 947908).

\bibliography{references}
\bibliographystyle{unsrtnat}


\clearpage
\appendix

\thispagestyle{empty}

\onecolumn \makesupplementtitle

\textbf{Outline.} The supplementary material is organized as follows. In \cref{secapp:A}, we prove general results on embeddings of losses, we compute the Bayes risks for each of the losses and we provide an algebraic characterization of the extreme points of the prediction sets. In~\cref{secapp:B} and \cref{secapp:C}, we provide the main results of the Max-Margin loss and the Restricted-Max-Margin loss, respectively.

\section{PRELIMINARY RESULTS}\label{secapp:A}

\subsection{Results on Embeddability of Losses}

\begin{proposition}\label{prop:bayesriskembedding}
Let $\psi:\calY\xrightarrow[]{}\Rspace{k}$ be an embedding of the output space. If $H_S = H_L$ and~$S(\psi(y)) = L_y$ for all $y\in\calY$, then $S$ embeds $L$ with embedding $\psi$.
\end{proposition}
\begin{proof}
To prove that $S$ embeds $L$ with embedding $\psi(y) = -L_y$, we need to show that 
\begin{equation*}
y\in y^\star(q) \iff \phi(y) \in v^\star(q).    
\end{equation*}
If $y\in y^\star(q)$, then
\begin{equation*}
    H_L(q) = ~L_y^\top q = S(\phi(y))^\top q = H(q) = \min_{v\in\Rspace{k}}~S(v)^\top q.
\end{equation*}
Thus, $S(\phi(y))^\top q = \min_{v\in\Rspace{k}}~S(v)^\top q$ implies that necessarily $\phi(y)\in v^\star(q)$. Similarly, if $\phi(y)\in v^\star(q)$, then $\min_{z}~L_z^\top q = L_y^\top q$ which implies $y\in y^\star(q)$.
\end{proof}

\subsection{Bayes risk identities}

The following \cref{lem:bayesrisks} provides an identity which will be useful to provide the forms of the Bayes risk for $S_{\M}$ and $S_{\RM}$.

\begin{lemma}\label{lem:bayesrisks} Let $\calC_y\subseteq\Delta$, $\Omega^y(q) = -L_y^\top q + i_{\calC_y}(q)$ and $S(v, y) = (\Omega^y)^*(v) - v_y$ for every $y\in\calY$. Then,
\begin{equation}\label{eq:lembayesrisks}
    H_S(q) = \max_{\substack{\sum_{y}q_y\nu_y = q \\ \nu_y\in\calC_y}}~\sum_{y} q_y L_y^\top \nu_y.
\end{equation}
\end{lemma}
\begin{proof}
Recall the definition of the Bayes risk $H(q) = \min_{v\in\Rspace{k}}~S(v)^\top q$. Using the structural assumption on $S$, we can re-write it as
\begin{equation*}
    H(q) = \min_{v\in\Rspace{k}}~\sum_{y\in\calY}q_y(\Omega^y)^*(v) - v^\top q = -\max_{v\in\Rspace{k}}~v^\top q - \sum_{y\in\calY}q_y(\Omega^y)^*(v) = -\Big(\sum_{y\in\calY}q_y(\Omega^y)^*\Big)^*(q).
\end{equation*}
Recall that if the functions $h_i$ are convex, then the conjugate of the sum is the infimum convolution of the individual conjugates \citep{rockafellar1997convex} as
\begin{equation*}
    \Big(\sum_ih_i\Big)^*(t) = \min_{\sum_ix_i = t}\sum_ih_i^*(x_i).
\end{equation*}
If we apply this property to the functions $h_i = q_i(\Omega^y)^*$, we obtain:
\begin{align*}
    -\Big(\sum_{y\in\calY}q_y(\Omega^y)^*\Big)^*(q) &= 
    -\min_{\sum_{y\in\calY}\nu_y = q}~\sum_{y\in\calY}(q_y(\Omega^y)^*)^*(\nu_y) \\
    &= -\min_{\sum_{y\in\calY}\nu_y = q}~\sum_{y\in\calY}q_y\Omega^y(\nu_y/q_y) && (ah)^*(x) = ah^*(x/a) \\
    &= -\min_{\substack{\sum_{y\in\calY}\nu_y = q \\ \nu_y/q_y\in\calC_y,~\forall y\in\calY}}~-\sum_{y\in\calY}L_y^\top \nu_y && \Omega^y(q) = -L_y^\top q + i_{\calC_y}(q) \\
    &= \max_{\substack{\sum_{y\in\calY}\nu_y = q \\ \nu_y/q_y\in\calC_y, ~\forall y\in\calY}}~\sum_{y\in\calY}L_y^\top \nu_y && \text{redefine }\nu_y\text{ as }\nu_y/q_y \\
    &= \max_{\substack{\sum_{y\in\calY}q_y\nu_y = q \\ \nu_y\in\calC_y, ~\forall y\in\calY}}~\sum_{y\in\calY} q_y L_y^\top \nu_y.
\end{align*}
\end{proof}
The following \cref{propapp:bayesrisks} provides us with the first part of Proposition 3.4.
\begin{proposition}[Bayes risks]\label{propapp:bayesrisks}
For all $q\in\Delta$, the Bayes risks read
\begin{align*}
    H_{\MM}(q) &= \min_{y\in\calY}~L_y^\top q = H_L(q) \\
    H_{\M}(q) &= \underset{Q\in U(q, q)}{\max}\langle L, Q\rangle_{\F} \\
    H_{\RM}(q) &= \underset{Q\in U(q, q)\cap \calC_L}{\max}\langle L, Q\rangle_{\F},
\end{align*}
where 
\begin{equation*}
    U(q, q) = \{Q\in\Rspace{k\times k}_{+}~|~Q1=q, Q^\top 1=q\}, \hspace{0.3cm}\text{and}\hspace{0.3cm}\calC_L = \{Q\in\Rspace{k\times k}~|~(1L_y^\top - L)Q_y\preceq 0,~\forall y\in\calY\}.
\end{equation*}
\end{proposition}
\begin{proof}
The first identity is trivial and has already been derived in the main body of the paper. We use the above \cref{lem:bayesrisks} to obtain the identities corresponding to $S_{\M}$ and $S_{\RM}$.
\begin{enumerate}
    \item \emph{Bayes risk of Max-Margin:} In this case $\calC_y = \Delta$. If we define $\Gamma\in\Rspace{k\times k}$ as the matrix whose rows are $\nu_y$, the maximization reads
    \begin{equation*}
        \max_{\substack{\Gamma^\top q = q \\ \Gamma 1 = 1 \\ \Gamma\succeq 0}}~\sum_{y\in\calY} q_yL_y^\top \Gamma_{y}
    \end{equation*}
    If we now define $Q\in\Rspace{k\times k}$ as 
    $Q = \operatorname{diag}(q)\Gamma$, i.e., $Q_y=q_y\Gamma_y$, the objective can be re-written as a matrix scalar product as
    \begin{equation*}
    \sum_{y\in\calY}q_yL_y^\top \Gamma_{y} = \sum_{y\in\calY}L_y^\top(q_y\Gamma_y) = \sum_{y\in\calY}L_y^\top Q_y = \langle L, Q\rangle_{\F}.
    \end{equation*}
    Whenever $q_y>0$, the change of variables $Q_y=q_y\Gamma_y$ is invertible and the constraints satisfy
    \begin{align}
        (Q^\top 1)_y = q_y &\iff (\Gamma^\top q)_y = q_y \\
        (Q1)_y = q_y &\iff (\Gamma 1)_y = 1 \\
        Q_y \succeq 0 &\iff \Gamma_y \succeq 0.
    \end{align}
    On the other hand, if $q_y=0$ then $Q_y = 0$ but the objective is not affected as it is independent of $\Gamma_y$.
    \item \emph{Bayes risk of Restricted-Max-Margin:} In this case $\calC_y = \Delta(y) = \{q\in\Delta~|~(L_y-L_{y'})^\top q \preceq 0, \forall y'\in\calY\}$. The maximization now reads
    \begin{equation*}
        \max_{\substack{\Gamma^\top q = q \\ \Gamma 1 = 1 \\ \Gamma\succeq 0 \\
        (1L_y^\top - L)\Gamma_y\preceq 0, \forall y }}~\sum_{y\in\calY} q_yL_y^\top \Gamma_{y}.
    \end{equation*}
    The result follows as $(1L_y^\top - L)\Gamma_y\preceq 0$ if and only if $(1L_y^\top - L)Q_y\preceq 0$ whenever $q_y>0$.
\end{enumerate}
\end{proof}

\subsection{Extreme points of a polytope}
\label{sec:extremepoints}
We will need to analyse the extreme points of the polytope $\calP = \{(q, u)\in\Rspace{k + 1}~|~q\in\Delta, L_y^\top q \geq u, \forall y\in\calY\}\subseteq\Rspace{k+1}$ in the proof of the sufficient condition for consistency of Max-Margin in \cref{sec:sufficientcondition}.

\textbf{Algebraic characterization of extreme points of a polyhedron. } The following \cref{prop:characterizationextreme} provides us with an algebraic characterization of the extreme points of a polyhedron $\calQ=\{x\in\Rspace{n}~|~Ax\succeq b\}$.
\begin{proposition}[Theorem 3.17 in \cite{andreasson2020introduction}]\label{prop:characterizationextreme} Let $x\in\calQ = \{x\in\Rspace{n}~|~Ax\succeq b\}$, where $A\in\Rspace{m\times n}$ has $\rank(A) = n$ and $b\in\Rspace{m}$.
Let $I\subseteq[m]$ be a set of indexes for which the subsystem is an equality, i.e., $A_Ix_I = b_I$ with~$Ax_I\succeq b$. Then $x_I$ is an extreme point of $\calQ$ if and only if $\rank(A_S) = n$.
\end{proposition}

Let $\calP\subseteq\Rspace{k+1}$ be the polyhedron defined as 
\begin{equation*}
    \calP = \{(q, u)\in\Rspace{k + 1}~|~q\in\Delta, L_y^\top q \geq u, \forall y\in\calY\}\subseteq\Rspace{k+1}.
\end{equation*}
The polyhedron $\calP$ can be written as $\calP = \{x=(q, u)\in\Rspace{k+1}~|~Ax\succeq b\}$ where
\begin{equation}\label{eq:inequalitiespolyhedron}
\begin{array}{l}
     \begin{array}{c}
         \\ \\ S \\ \\
     \end{array}  \\
     \begin{array}{c}
         \\ \\ T \\ \\
     \end{array} \\
     \begin{array}{c}
         \\ \\
     \end{array}
\end{array}
\begin{array}{l}
    \left\{
     \begin{array}{c}
         \\ \\ \\ \\
     \end{array}\right.  \\
     \left\{
     \begin{array}{c}
         \\ \\ \\ \\
     \end{array}\right. \\
     \begin{array}{c}
         \\ \\
     \end{array}
\end{array}
    \underbrace{\left(
    \begin{array}{c | c}
        \begin{array}{ccc}
         & &  \\
         & \text{\Huge{L}} & \\ 
         & &  \\
    \end{array} & \begin{array}{c}
         \vdots  \\
         -1 \\ 
         \vdots \\
    \end{array}  \\ \hline
    \begin{array}{ccc}
         & &  \\
         & \text{\Huge{Id}} & \\
         & & \\
    \end{array} & \begin{array}{c}
         \vdots  \\
         0 \\ 
         \vdots \\
    \end{array} \\ \hline 
        \begin{array}{ccc}
            \cdots & 1 & \cdots  \\
        \end{array} & 0 \\
        \begin{array}{ccc}
            \cdots & -1 & \cdots  \\
        \end{array} & 0 \\
    \end{array}
    \right)}_{A}
    \left(\begin{array}{c}
        \\
         \\
         q  \\
         \\ 
         \\ \hline 
         u
    \end{array}\right) \succeq
    \underbrace{\left(\begin{array}{c}
          \vdots\\
          0 \\
          \vdots \\ \hline
          \vdots \\
          0 \\
          \vdots
          \\ \hline 
         1 \\
         -1 \
    \end{array}\right)}_{b},
\end{equation}
with $A\in\Rspace{(2k+2)\times (k+1)}$ and $b\in\Rspace{k+1}$.
Note that $\operatorname{rank}(A) = k + 1$.
Given $x=(q, u)$, define~$S,T\subseteq\calY$ as the subsets of outputs such that 
\begin{equation}\label{eq:definitionSandT}
    y\in S \iff L_y^\top q = u, \hspace{0.5cm} y\in T \iff q_y = 0,
\end{equation}
i.e., $S$ and $T$ correspond to the indexes of the first and second block of the matrix $A$ for which the inequality holds as an equality, respectively. More concretely, if $I$ are the indices of $A$ for which~$A_Ix_I = b_I$, we have that 
$I = S \cup( k + T )\cup \{2k+1\}\cup\{2k + 1\}$, because the last two inequalities must be an equality as $q\in\Delta$. Moreover, the sets $S$ have the following properties:
\begin{itemize}
    \item[-] We necessarily have $|S|\geq 1$: if $S = \emptyset$, then $\operatorname{rank}(A_I) = k$ and so the rank is not maximal, thus $x = (q,u)$ cannot be an extreme point. 
    \item[-] We necessarily have $|S| + |T| \geq k$ (using the fact that $\operatorname{rank}(A) = k + 1$).
\end{itemize}

\newpage
\section{RESULTS ON MAX-MARGIN LOSS}\label{secapp:B}
\subsection{Bayes Risk of Max-Margin for Symmetric Losses}
The following \cref{propapp:maxmarginbayesrisks} gives another expression of the Bayes risk of~$S_{\M}$ and its Fenchel conjugate assuming the loss~$L$ is symmetric.

\begin{proposition}\label{propapp:maxmarginbayesrisks}
Let $H_{\M}(q) = \max_{Q\in U(q, q)}~\langle L, Q\rangle_{\F}$ and assume $L$ symmetric. Then, the following identities hold:
\begin{align}
    H_{\M}(q) &= \min_{\frac{1}{2}(a_y+a_{y'})\geq L(y, y')}~a^\top q, \hspace{0.5cm}\forall q\in\Delta, \\
    (-H_{\M})^*(v) &= \max_{y, y'\in\calY}~L(y, y') + \frac{v_{y} + v_{y'}}{2}, \hspace{0.5cm} \forall v\in\Rspace{k}.
\end{align}
\end{proposition}
\begin{proof} The first part corresponds to the dual of the maximization problem defining the Bayes risk $H_{\M}$ when $L$ is symmetric:
\begin{align*}
    -\min_{Q\in U(q, q)}-\langle L, Q\rangle_{\F} &= \min_{Q\succeq 0}\max_{a,b\in\Rspace{k}}~a^\top (Q1 - q) + b^\top(Q^\top 1 - q) - \langle L, Q\rangle_{\F} \\
    &= -\max_{a,b\in\Rspace{k}}-a^\top q - b^\top q + \min_{Q\succeq 0}~a^\top Q 1 + b^\top Q^\top 1 - \langle L, Q\rangle_{\F}
\end{align*}
We can now re-write the minimization objective as a matrix scalar product with $Q$ as $a^\top Q 1 = \operatorname{Tr}(a^\top Q 1) = \operatorname{Tr(Q1 a^\top)} = \langle Q, a1^\top\rangle_{\F}$ and analogously $b^\top Q 1 = \langle Q, 1b^\top\rangle_{\F}$. Hence, the objective of the minimum becomes $\langle a1^\top + 1b^\top - L, Q\rangle_{\F}$, which gives
\begin{equation*}
    \min_{Q\succeq 0}~\langle a1^\top + 1b^\top - L, Q\rangle_{\F} = \left\{\begin{array}{ll}
        0 & \text{if}~ a1^\top + 1b^\top - L \succeq 0 \\
        -\infty &  \text{otherwise}
    \end{array}\right. .
\end{equation*}
We obtain the following minimization problem in $a,b\in\Rspace{k}$
\begin{equation*}
    -\max_{a1^\top + 1b^\top\succeq L}-(a+b)^\top q = \min_{a1^\top + 1b^\top\succeq L}~(a+b)^\top q.
\end{equation*}
Using that $L$ is symmetric, we can add the constraint $a = b$. In order to see this, let $(a^\star, b^\star)$ be a solution of the linear problem. If $L$ is symmetric, then $(b^\star, a^\star)$ is also a solution, which implies that $\frac{1}{2}(a^\star + b^\star, a^\star + b^\star)$ too. Hence, we can assume $a = b$ and we obtain the desired result.

For the second part, note that if $L$ is symmetric, the matrix $Q$ can be assumed also symmetric. To see this, let $Q^\star = \argmax_{Q\in U(q, q)}\langle L, Q\rangle_{\F}$. Then if $L$ symmetric $(Q^\star)^\top$ is also a solution, which means that $\frac{1}{2}(Q^\star + (Q^\star)^\top)$ too, which is symmetric. Hence, we can write 
\begin{align*}
    H_{\M}(q) &= \max_{\substack{Q=Q^\top \\ Q1=q \\ Q\succeq 0}}~\langle L, Q\rangle_{\F} \\
            &= \min_{v\in\Rspace{k}}\max_{\substack{Q = Q^\top \\ Q\in \operatorname{Prob}(\calY\times\calY)}}~\langle L, Q\rangle_{\F} - v^\top (Q1 - q) \\
            &= \min_{v\in\Rspace{k}}\Big\{\underbrace{\max_{\substack{ Q = Q^\top \\ Q\in \operatorname{Prob}(\calY\times\calY) }}~\langle L + v1^\top, Q\rangle_{\F}}_{(-H_{\M})^*(v)}\Big\} - v^\top q,
\end{align*}
where at the last step we have used that $q\in\Delta$ and so $1^\top Q1=1$, which together with $Q\succeq 0$ implies $Q\in \operatorname{Prob}(\calY\times\calY)$. The extreme points of the problem domain $\{Q\in\operatorname{Prob}(\calY\times\calY)\}$ where the maximization of the linear objective is achieved are precisely the points $\{\frac{1}{2}(e_y + e_{y'})\}_{y,y'\in\calY}$.
\end{proof}

\subsection{Necessary Conditions for Consistency}






Recall that $S$ is consistent to $L$ if there exists a decoding $d:\Rspace{k}\xrightarrow[]{}\calY$ such that if $v\in v^\star(q)$, then necessarily~$d(v)\in y^\star(q)$ for all $q\in\Delta$. A necessary condition for this to hold is that every level set of $v^\star$ must be included in a level set of $y^\star$, which are precisely the prediction sets. 

\begin{lemma}
If $S$ is consistent to $L$, then for every $v\in\Rspace{k}$ there must exist a $y\in\calY$ such that
\begin{equation}\label{eq:01}
    (v^\star)^{-1}(v) \subseteq (y^\star)^{-1}(y) = \Delta(y).
\end{equation}
\end{lemma}
\begin{proof}
If \eqref{eq:01} does not hold, then there exists $q_1,q_2\in (v^\star)^{-1}(v)$ with $y^\star(q_1)\cap y^\star(q_2) = \emptyset$. However, Fisher consistency means
that $v\in v^\star(q_1)$ implies $d(v)\in y^\star(q_1)$ and $v\in v^\star(q_2)$ implies $d(v)\in y^\star(q_2)$, which is not possible because~$y^\star(q_1)\cap y^\star(q_2) = \emptyset$.
\end{proof}

The following \cref{prop:necessaryconsistency} re-writes \eqref{eq:01} in terms of the Bayes risk $H_S$.

\begin{proposition}\label{prop:vstarlevelsets}
The level sets of $v^\star$ are the image of $-\partial(-H_S)^*:\Rspace{k}\xrightarrow[]{}2^{\Delta}$, i.e., 
\begin{equation}
    \operatorname{Im}((v^\star)^{-1}) = \operatorname{Im}(-\partial(-H_S)^*).
\end{equation}
\end{proposition}
\begin{proof}
First of all, note that $-\partial(-H_S)^* = (\partial H_{\M})^{-1}$ \citep{rockafellar1997convex}. We have
\begin{equation*}
H_S(q) = \min_{v\in\Rspace{k}}~S(v)^\top q + i_{\Delta}(q), \hspace{0.5cm} \partial H_S(q) = S(\bar{v}) + \langle 1\rangle, ~\bar{v}\in v^\star(q).
\end{equation*}
Let's now prove the two inclusions.

($\subseteq$): Let $Q\in \operatorname{Im}((v^\star)^{-1})$. This means that there exists $V\in\Rspace{k}$ such that $V = \argmin_{v\in\Rspace{k}}~S(v)^\top q$ for all $q\in Q$. If we define $T = S(V) + \langle 1\rangle$, then $T = \partial H_{\M}(q)$ for all $q\in Q$.

($\supseteq$): Let $Q\in \operatorname{Im}((\partial H_{\M})^{-1})$. This means that there exists $T$ such that $T = \partial H_{\M}(q)$ for all $q\in Q$. For every $q\in Q$, the set $T$ can be written as $T = S(v^\star(q)) + \langle 1\rangle$. To show that $v^\star(q) = v^\star(q')$ for all $q, q'\in Q$, we need to show that if $S(v) = S(v') + c1, v\in v^\star(q),v'\in v^\star(q')$ for some $q,q'\in Q$, then necessarily $c=0$. This is because $S(v(q))^\top q'\geq S(v(q'))^\top q'\implies c\geq 0$ and $S(v(q))^\top q\leq S(v(q'))^\top q\implies c \leq 0$.



\end{proof}
\begin{corollary}[Necessary condition for consistency]\label{prop:necessaryconsistency}
If $S$ is Fisher consistent to $L$, then for every~$v\in\Rspace{k}$, there exists~$y\in\calY$ such that
\begin{equation}\label{eq:bayesrisksconsistency}
    -\partial(-H_{S})^*(v)\subseteq\Delta(y).
\end{equation}
\end{corollary}
\begin{proof}
This follows directly from \cref{eq:01} and \cref{prop:vstarlevelsets}.
\end{proof}

\begin{proposition}[Weaker necessary condition for consistency]\label{prop:weakerconsistency} If $S$ is consistent to $L$, then every extreme point of $\Delta(y)$ for some $y\in\calY$ must be a 0-dimensional image of~$-\partial(-H_{S})^*$.
\end{proposition}
\begin{proof}
Let $\Delta_S(v) = -\partial(-H_{S})^*(v)$. There exists a finite set $\calV\subseteq\Rspace{k}$ such that $\bigcup_{v\in\calV}\Delta_S(v) = \Delta(y)$. In particular, if $q$ is an extreme point of $\Delta(y)$, then there exists $v\in\calV$ such that $q\in \Delta_S(v)$. We need to show that $q$ is also an extreme point of $\Delta_S(v)$. Indeed, if $\Delta_S(v)\subseteq\Delta(y)$ are polyhedrons and $q\in \Delta_S(v), \Delta(y)$ is an extreme point of $\Delta(y)$, then it is also necessarily an extreme point of~$\Delta_S(v)$.
\end{proof}

\begin{theorem}\label{th:mainresultmaxmarginapp}
Let $L$ be a symmetric loss with $k>2$. If the Max-Margin loss is consistent to $L$, then $L$ is a distance, and for every three outputs $y_1,y_2,y_3\in\calY$, there exists $z\in\calY$ for which these the following three identities are satisfied:
\begin{align*}
    L(y_1, y_2) &= L(y_1, z) + L(z, y_2), \\
    L(y_1, y_3) &= L(y_1, z) + L(z, y_3), \\
    L(y_2, y_3) &= L(y_2, z) + L(z, y_3).
\end{align*}
\end{theorem}

\begin{proof}
\begin{figure}[ht!]
    \centering
    \includegraphics[width=0.21\textwidth]{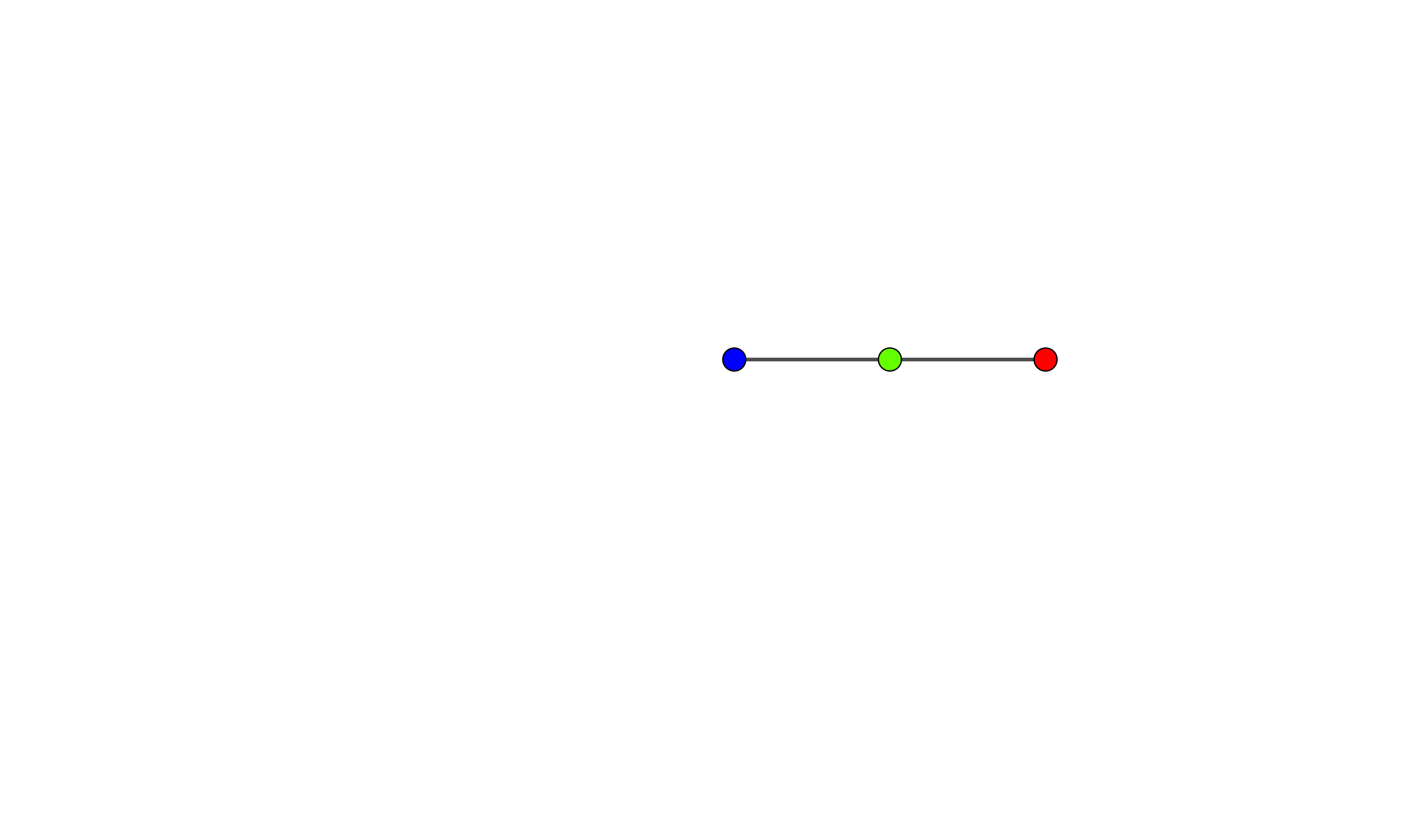}
    \includegraphics[width=0.24\textwidth]{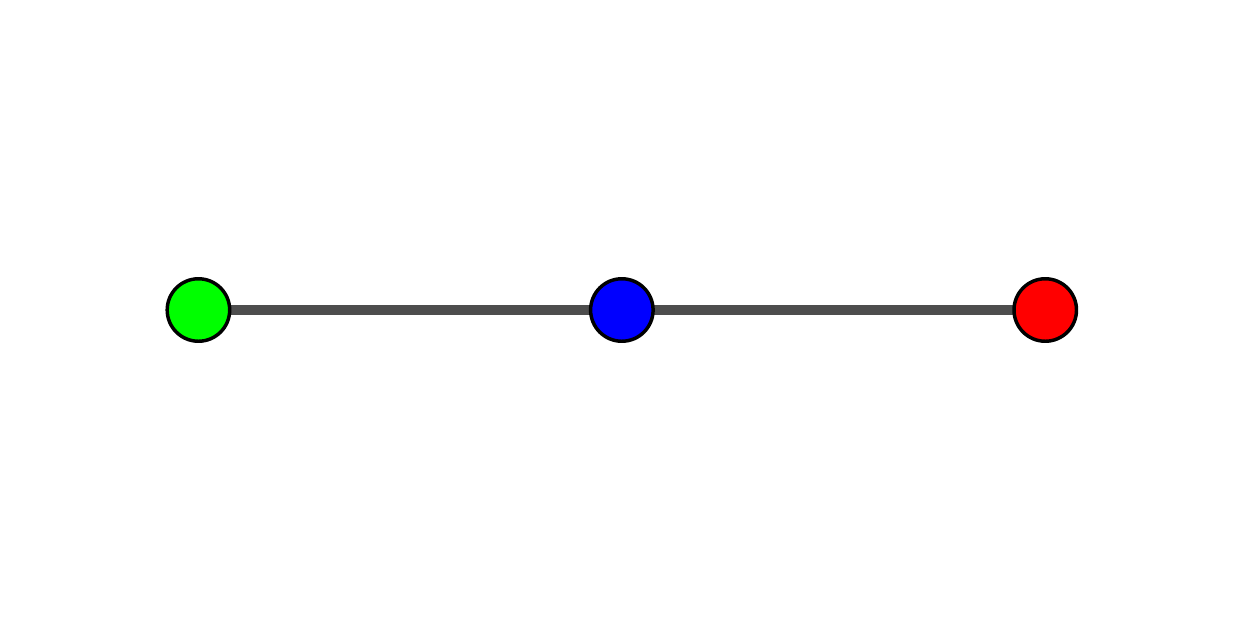}
    \includegraphics[width=0.24\textwidth]{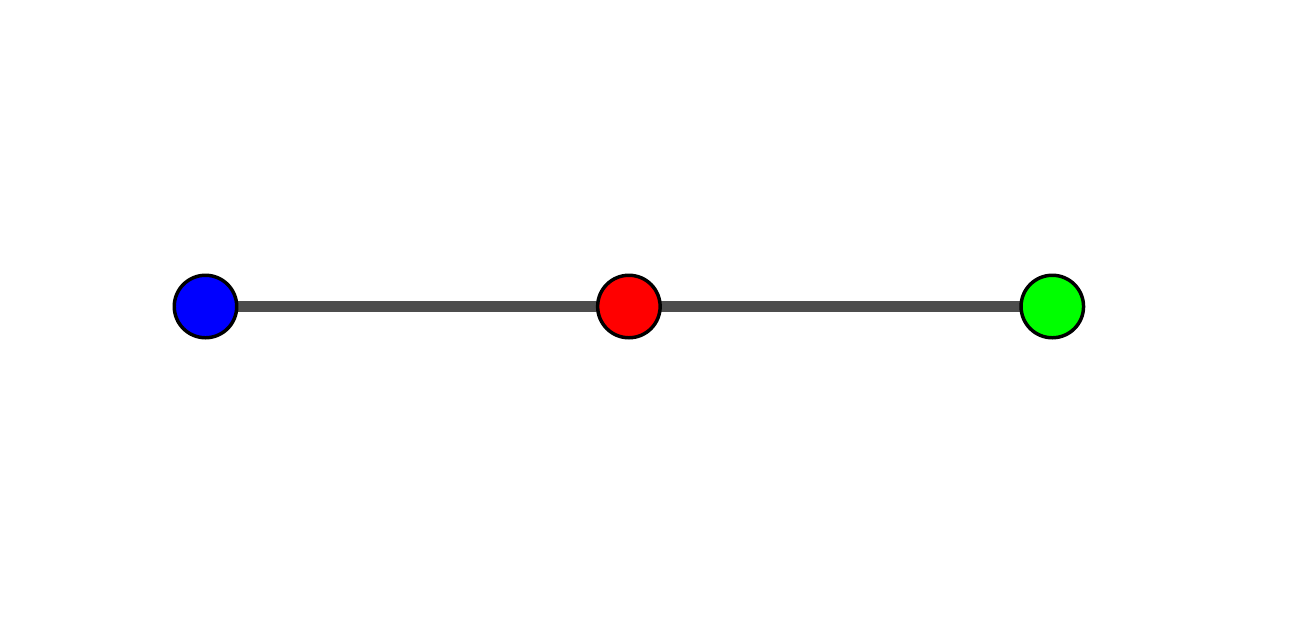}
    \includegraphics[width=0.24\textwidth]{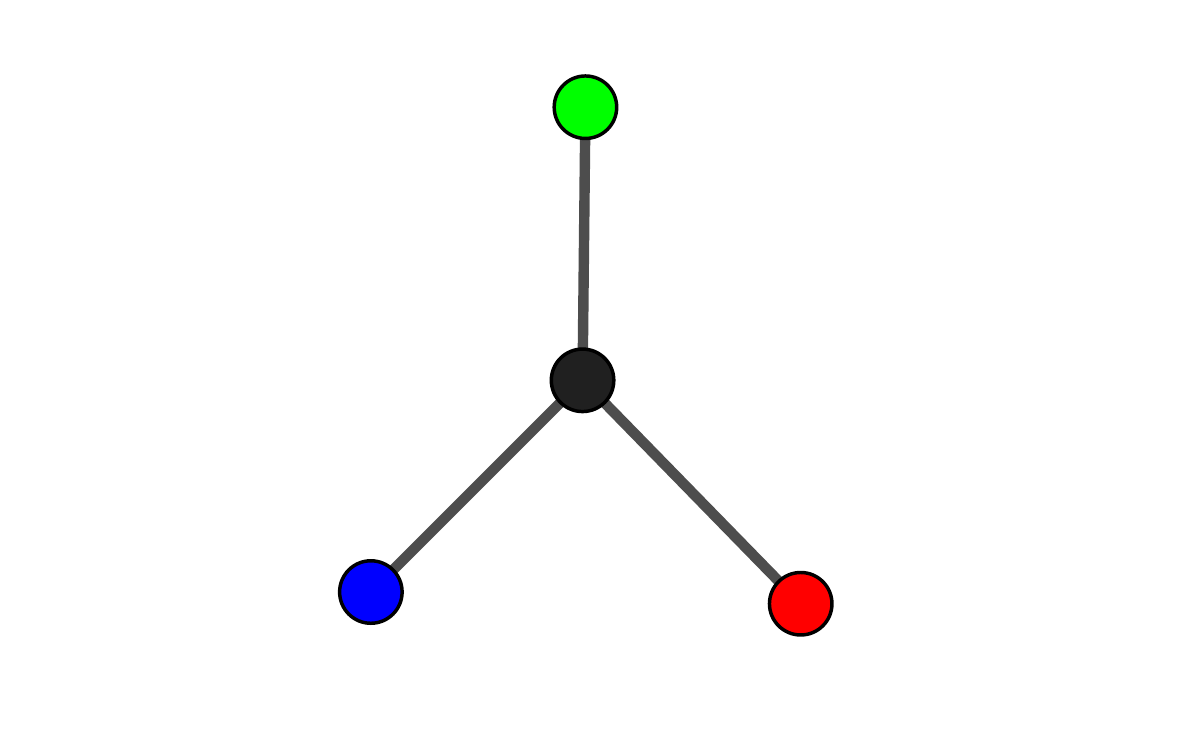}
    \includegraphics[width=0.24\textwidth]{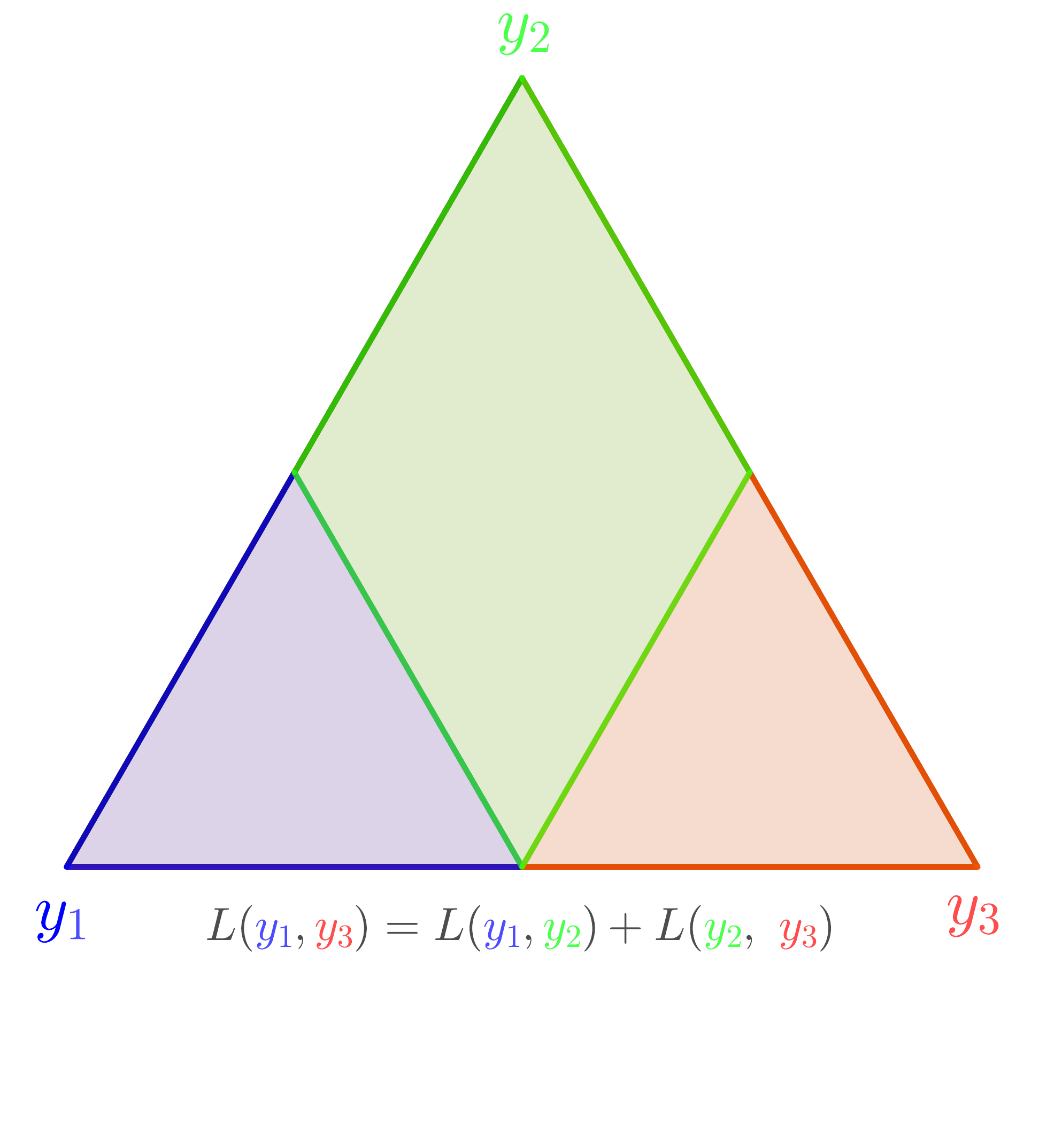}
    \includegraphics[width=0.24\textwidth]{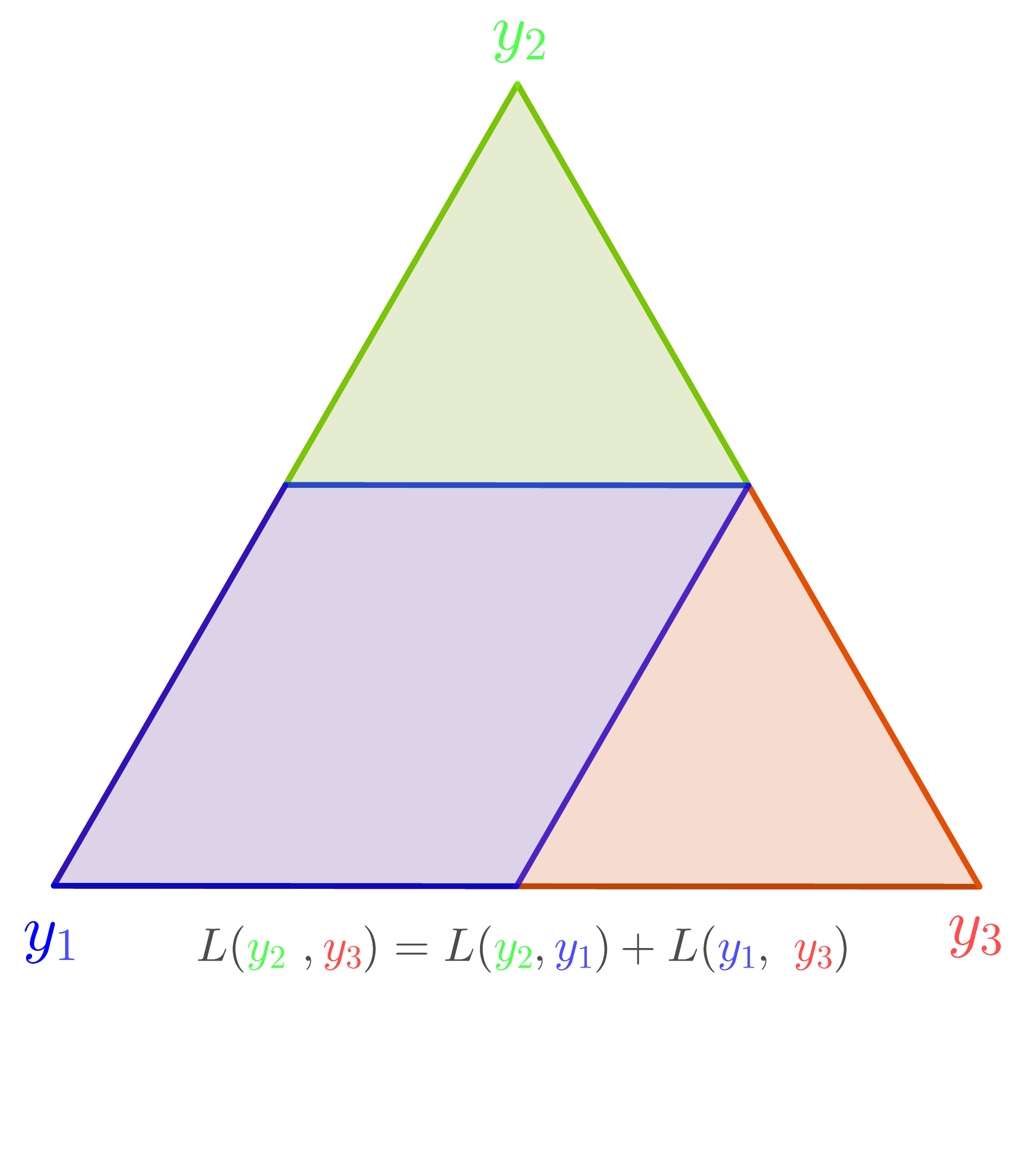}
    \includegraphics[width=0.24\textwidth]{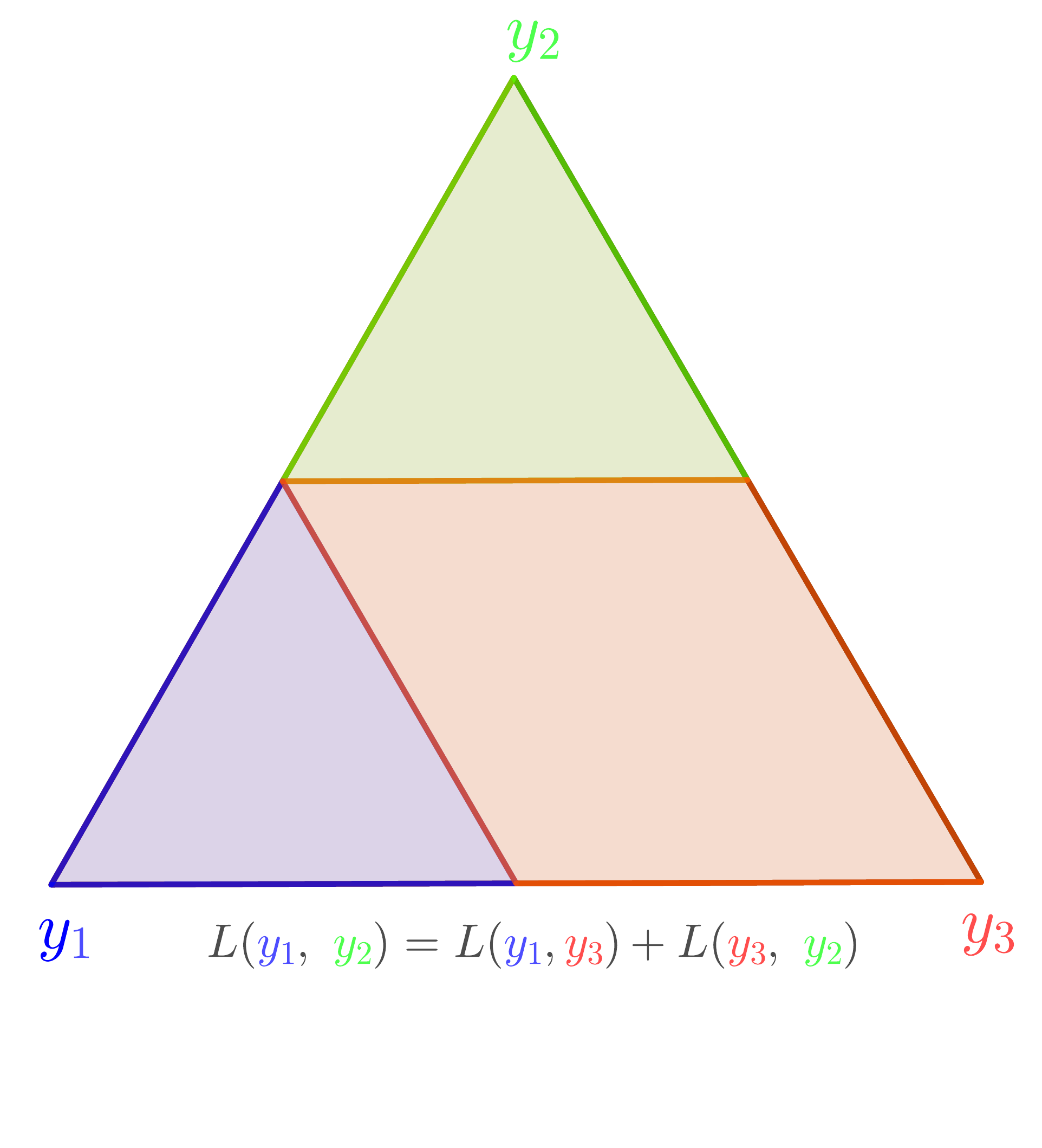}
    \includegraphics[width=0.26\textwidth]{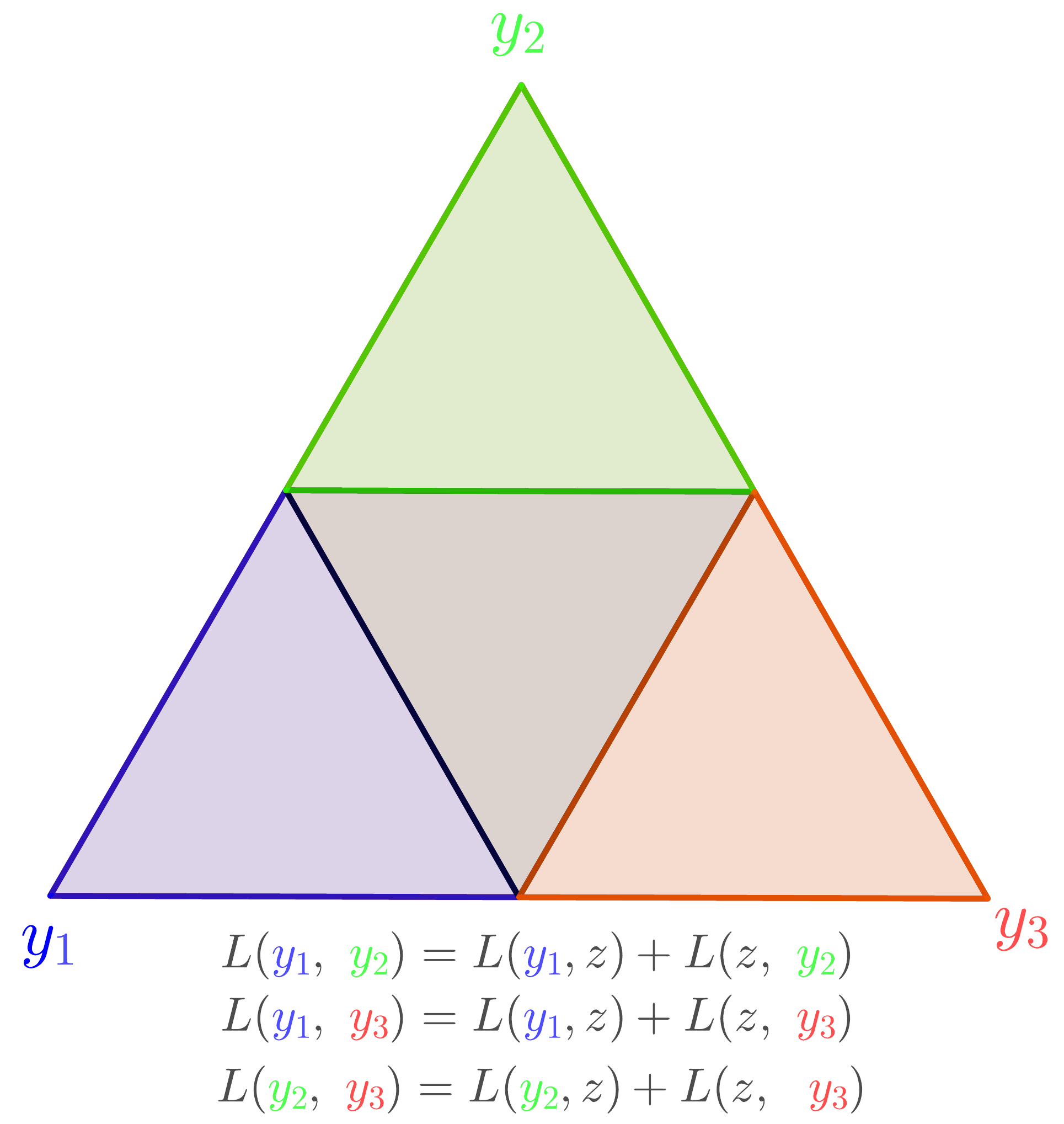}
    \caption{These are the only possible possibilities of the prediction sets in a three-dimensional face of the simplex. The equations associated to each configuration is written below the corresponding simplex and all together can be compactly written as the necessary condition given by the theorem. An edge from a corner of the simplex to the middle point of the opposite side is not possible as $L(y,y')=0$ if and only if $y=y'$ by assumption.}
    \label{fig:facetpossibilities}
\end{figure}

From \cref{propapp:maxmarginbayesrisks} and \cref{prop:weakerconsistency}, we obtain that if the Max-Margin loss is consistent to $L$, then the extreme points of the prediction sets $\Delta(y)$'s have to be of the form $1/2(e_y + e_{y'})$. Hence, the projection of the sets $\Delta(y)$'s into a three-dimensional simplex can only be of the form depicted in \cref{fig:facetpossibilities}. The necessary condition follows directly from these possibilities (see caption of \cref{fig:facetpossibilities}). Moreover, note that if the three identities of the theorem hold, then $L$ is a distance. To see that, note that the triangle inequality holds for any triplet $y_1, y_2, y_3\in\calY$ as:
\begin{align*}
    L(y_1,y_2) &= L(y_1, z) + L(z, y_2) \\
    &= L(y_1, y_3) - L(z, y_3) + L(y_3, y_2) - L(y_3, z) \\
    &= L(y_1, y_3) + L(y_3, y_1) - 2L(y_3, z) \\
    &\leq L(y_1, y_3) + L(y_3, y_1).
\end{align*}
\end{proof}

\paragraph{Examples of losses not satisfying the necessary condition. } We now show that the examples exposed in Section 2 do not satisfy the necessary condition of \cref{th:mainresultmaxmarginapp}.

\begin{lemma}
Let $L$ such with full rank loss matrix $L$ and existing $q\in\operatorname{int}(\Delta)$ for which all outputs optimal, i.e., $y^\star(q)=\calY$. Then, the Max-Margin loss is \emph{not} consistent to $L$.
\end{lemma}
\begin{proof}
The point $q$, which is not of the form $1/2(e_y+e_{y'})$ for some $y,y'\in\calY$, is an extreme point of the polytope $\Delta(y) = \{q\in\Delta~|~L_z^\top q \geq L_y^\top q, \forall z\in\calY\}$ for every $y\in\calY$. This is because $q\in\Delta$ is the unique solution of $Lq=u$ with $u=L_y^\top q$ for all $y\in\calY$. Hence, by \cref{prop:weakerconsistency}, the Max-Margin loss is not consistent to $L$.
\end{proof}

\begin{lemma}
The Max-Margin loss is not consistent to the the Hamming loss on permutations $L(\sigma,\sigma')=\frac{1}{M}\sum_{m=1}^M1(\sigma(m)\neq\sigma'(m))$ where $\sigma,\sigma'$ permutations of size $M$.
\end{lemma}
\begin{proof}
Take the transpositions $\sigma_1 = (3, 2), \sigma_2=(2, 1), \sigma_3= (3, 1)$. We have that $L(\sigma_i, \sigma_j) = 2/M$ for $i\neq j$ and $L(\sigma, \sigma')>\frac{2}{M}$ for all permutations $\sigma\neq \sigma'$. Hence, the necessary condition can't be satisfied. 
\end{proof}

\paragraph{The Hamming loss with $\boldsymbol{M=k=2}$ is consistent and it is not defined in a tree.} The Hamming loss $L(y,y') = \frac{1}{2}(1(y_1\neq y_1') + 1(y_2\neq y_2'))$ is consistent as it decomposes additively and each term is consistent as it is the binary 0-1 loss. However, it can't be described as the shortest path distance in a tree, but rather the shortest path distance in a cycle of size four with all weights equal to $1/2$.

\begin{figure}[ht!]
    \centering
    \includegraphics[width=0.45\textwidth]{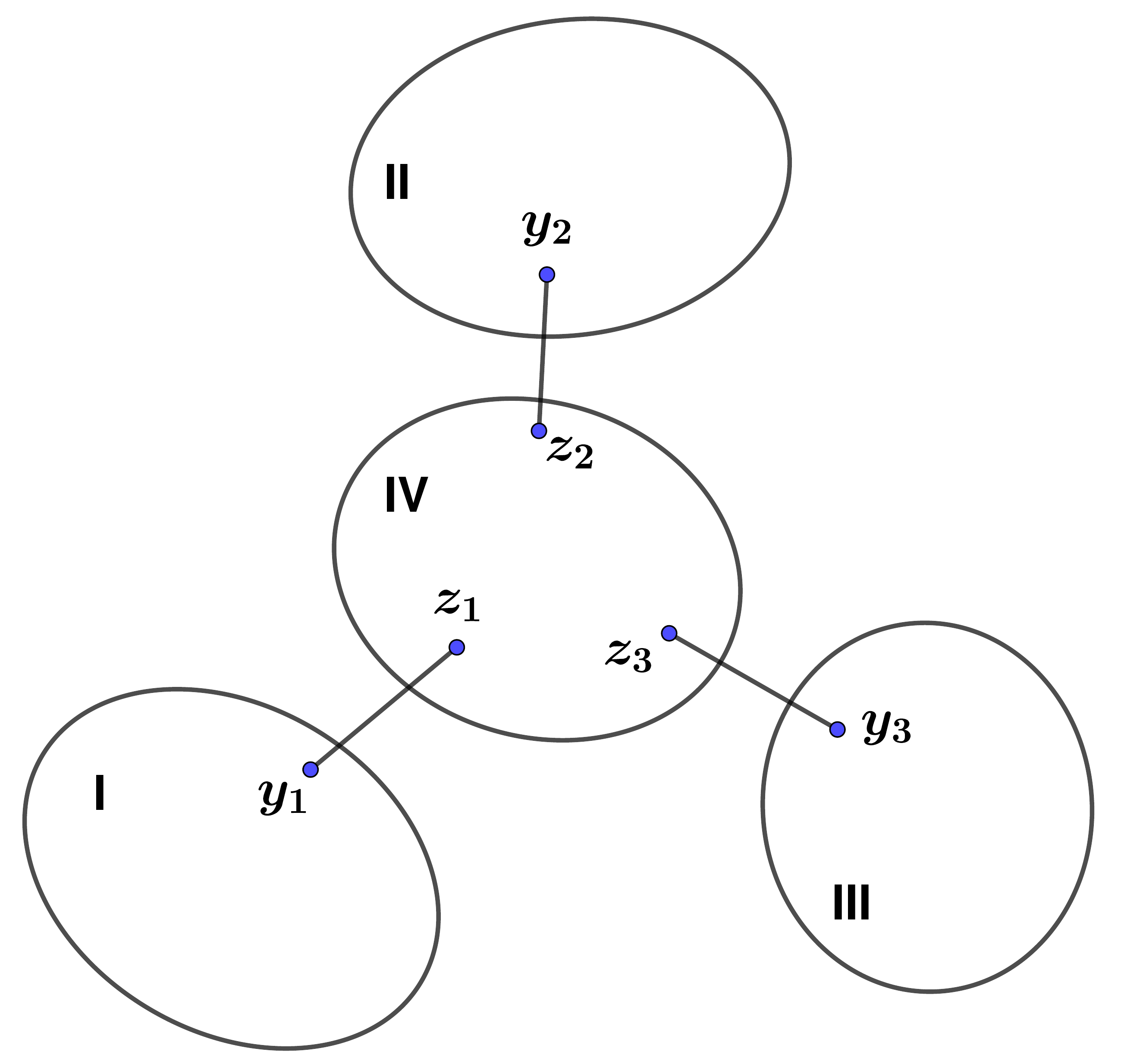}
    \includegraphics[width=0.5\textwidth]{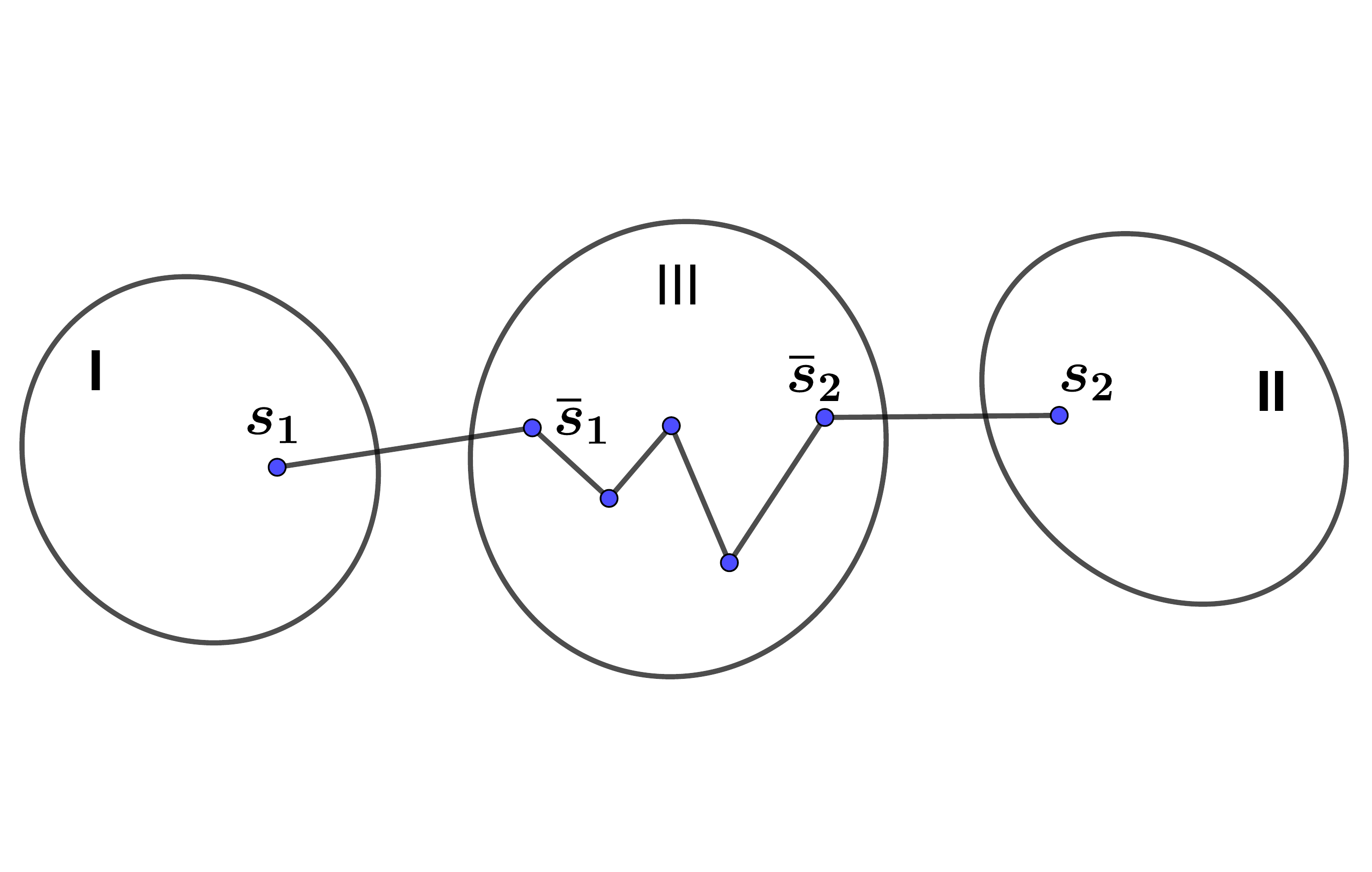}
    \caption{\textbf{Left:} The 4-partition of the output space where $y_1,y_2,y_3\in S$. The different sets only communicate through the edges $z_i-y_i$ for $i=1,2,3$ respectively. \textbf{Right:} The 3-partition of the output space where $s_1,s_2\in S$ and the depicted path is the path between these two points.}
    \label{fig:treeproof}
\end{figure}

\subsection{Sufficient condition on the discrete loss L}
\label{sec:sufficientcondition}

\begin{theorem}
If $L$ is a distance defined in a tree, then the Max-Margin loss $S_{\M}$ embeds $L$ with embedding $\psi(y) = -L_y$ and it is consistent under the argmax decoding. 
\end{theorem}
\begin{proof}

We just have to prove that the extremes points of the polytope $\calP$ defined as
\begin{equation*}
    \calP = \{(q, u)\in\Rspace{k + 1}~|~q\in\Delta, L_y^\top q \geq u, \forall y\in\calY\}\subseteq\Rspace{k+1}
\end{equation*}

are of the form $(\frac{1}{2}(e_y + e_{y'}), \frac{1}{2}L(y, y'))$, where $y,y'\in\calY$. Using this, we obtain
\begin{align*}
    (-H_{2L})^*(v) &= \max_{q\in\Delta}\min_{z\in\calY}~2L_z^\top q + v^\top q \\
    &= \max_{(q, u)\in\calP}~2u + v^\top q \\
    &= \max_{(q, u)\in\operatorname{ext}(\calP)}~2u + v^\top q \\
    &= \max_{y,y'\in\calY}~L(y,y') + \frac{v_y + v_{y'}}{2} = (-H_{\M})^*(v),
\end{align*}
for all $v\in\Rspace{k}$. This implies $H_{2L} = 2H_L = H_{\M}$. We will use the algebraic framework introduced in \cref{sec:extremepoints}.

Let $x\in\operatorname{ext}(\calP)$ and $S$ are $T$ the sets of indices for which $L_y^\top q=u$ and $q_{y'} = 0$ for $y\in S$ and $y'\in T$ (as defined in \eqref{eq:definitionSandT}).

If $|S|=1$, then we necessarily. have $|T| = k-1$ because $|T|\geq k-|S| = k-1$ and $|T|=k$ is not possible because $q$ is in the simplex. In this case, the extreme point is equal to $q = (e_y, 0)$, which is of the desired form.

\textbf{First part of the proof. } If $|S|\geq 2$, let's prove that the elements in $S$ must be necessarily aligned, i.e., contained in a chain (always true for $|S|=2$). If we denote by $\operatorname{SP}(s, s')\subseteq\calY$ the elements in the shortest path between $s, s'$, this means that there exists $s_1, s_2\in S$ such that $s\in \operatorname{SP}(s_1, s_2)$ for all $s\in S$. 

If the elements in $S$ are not aligned, then there exist pairwise different elements $y_1,y_2,y_3\in S$ and $z_1,z_2,z_3\in\calY$ (possibly repeated) such that the tree defining the loss $L$ can be partitioned into four sets $\I, \II, \III, \IV$ of the form depicted in the left \cref{fig:treeproof}, where the edges $y_i-z_i$ belong to the tree for $i=1,2,3$.

\begin{align*}
    L_{z_1}^\top q &= \sum_{y\in\I}~L(z_1, y)q_y + \sum_{\substack{y\notin\I \\y\neq z_1}}~L(z_1, y)q_y \\
    &= \sum_{y\in\I}~(L(z_1, y_1) + L(y_1, y))q_y + \sum_{\substack{y\notin\I \\y\neq z_1}}~(L(y_1, y) - L(z_1, y_1))q_y \\
    &= \sum_{y\neq z_1}~L(y_1, y)q_y + \Big(\sum_{y\in\I}q_i - \sum_{\substack{y\notin\I \\i\neq z_1}}q_y\Big)L(z_1, y_1) \\
    &= \sum_{y\in\calY}~L(y_1, y)q_y + \Big(\sum_{y\in\I}q_i - \sum_{y\notin\I}q_y\Big)L(z_1, y_1)\\
    &= u + \Big(\sum_{y\in\I}q_y - \sum_{y\notin\I}q_y\Big)L(z_1, y_1).
\end{align*}

If we repeat the same procedure for $\II$ and $\III$, we obtain that 
\begin{equation*}
    \left\{
    \begin{array}{ll}
        L_{z_1}^\top q &= u + \Big(\sum_{y\in\I}q_y - \sum_{y\notin\I}q_y\Big)L(z_1, y_1) \\
        L_{z_2}^\top q &= u + \Big(\sum_{y\in\II}q_y - \sum_{y\notin\II}q_y\Big)L(z_2, y_2) \\
        L_{z_3}^\top q &= u + \Big(\sum_{y\in\III}q_y - \sum_{y\notin\III}q_y\Big)L(z_3, y_3)
    \end{array}\right.
\end{equation*}
However, note that $\sum_{y\in\I}q_y +\sum_{y\in\II}q_y +\sum_{y\in\III}q_y +\sum_{y\in\IV}q_y =1$, which implies that 
\begin{equation*}
    \min_{\calA\in\{\I,\II,\III\}}~\Big(\sum_{y\in \calA}q_y - \sum_{y\notin\calA}q_y\Big) < 0.
\end{equation*}
Hence, there exists $i\in\{1,2,3\}$ for which $L_{z_i}^\top q < u$, which leads to a contradiction as $L_y^\top q \geq u$ for all $y\in\calY$.

\textbf{Second part of the proof.} Now that we have that the elements in $S$ must be aligned, let's proceed with the proof by analyzing separately particular cases:
\begin{itemize}
    \item[-] ($\boldsymbol{S\cap T=\emptyset}$): This means that $q_s>0$ for all $s\in S$. Let $x=(\frac{1}{2}(e_{s_1} + e_{s_2}), \frac{1}{2}L(s_1, s_2))$, where $S\subseteq\operatorname{SP}(s_1, s_2)$. Then, it satisfies the equality constraints as $L_sq = 1/2(L(s_1, s) + L(s, s_2)) = 1/2L(s_1, s_2)$ because $s\in\operatorname{SP}(s_1,s_2)$ for all $s\in S$. Hence, it has to be equal to the unique solution of the linear system of equations.
    \item[-] ($\boldsymbol{S\cap T\neq\emptyset}$): Let's separate into two more cases.
    \begin{itemize}
        \item ($\boldsymbol{\exists r_1,r_2\in[k]\setminus T}$ \textbf{such that} $\boldsymbol{S\subseteq\operatorname{SP}(r_1,r_2)}$): Let $x=(\frac{1}{2}(e_{r_1} + e_{r_2}), \frac{1}{2}L(r_1, r_2))$. Then, it satisfies the equality constraints as $L_sq = 1/2(L(r_1, s) + L(s, r_2)) = 1/2L(r_1, r_2)$ because $s\in\operatorname{SP}(r_1,r_2)$ for all $s\in S$. Hence, it has to be equal to the unique solution of the linear system of equations.
        \item ($\boldsymbol{\nexists r_1,r_2\in[k]\setminus T}$ \textbf{such that} $\boldsymbol{S\subseteq\operatorname{SP}(r_1,r_2)}$): We will show that this case is not possible. Consider the shortest path between $s_1$ and $s_2$ in $S$ and the partition of the vertices of the tree into the sets $\I,\II,\III$ depicted in the right \cref{fig:treeproof}. We know that 
        \begin{equation*}
            \{\I\cap([k]\setminus T)=\emptyset\}\vee\{\II\cap([k]\setminus T)=\emptyset\}.
        \end{equation*}
        If this is not true, then taking $r_1\in\I$ and $r_2\in\II$ we obtain $S\subseteq\operatorname{SP}(r_1, r_2)$. Assume that $\I\cap([k]\setminus T)=\emptyset$. We have that $L(s_1, y) > L(\bar{s}_1, y)$ for all $y\in[k]\setminus T$, which means that \begin{equation*}
            L_{s_1}^\top q > L_{\bar{s}_1}^\top q.
        \end{equation*}
        This is a contradiction because $L_{y}^\top q \geq u = L_{s_1}^\top q$ as $s_1\in S$. The case $\II\cap([k]\setminus T)=\emptyset$ can be done analogously. 
    \end{itemize}
\end{itemize}

\textbf{Third part of the proof.} By \cref{prop:bayesriskembedding}, to prove that $S_{\M}$ embeds $2L$ with embedding $\psi(y) = -L_y$, we only need to show that $S_{\M}(-L_y) = 2L_y$. For every $z\in\calY$, we have
\begin{align*}
    S_{\M}(-L_y, z) &= \max_{y'\in\calY}~L(z, y') + (-L_y)^\top e_{y'} - (-L_y)^\top e_z \\
    &= \max_{y'\in\calY}~\{L(z, y') - L(y, y')\} + L(y, z) \\
    &= 2L(y, z),
\end{align*}
where at the last step we have used $L(z, y') - L(y, y') \leq L(y, y')$ as $L$ is a distance and so the maximization is achieved at $y'=y$. 

Finally, the argmax decoding is consistent as it is an inverse of the embedding $\psi(y) = -L_y$ as 
\begin{equation*}
d(\psi(y)) = \argmax_{y'\in\calY}~-L(y, y') = \argmin_{y'\in\calY}~L(y, y') = y.
\end{equation*}
\end{proof}

\subsection{Partial Consistency through dominant label condition}

\begin{lemma}\label{lemapp:distancecorners}
Let $q\in\Delta$ such that $q_1\geq 1/2\geq p_y$ for all $y\neq 1$. If $L$ is a distance, then $L_1^\top q \leq L_{y}^\top q$ for all $y\in\calY$.
\end{lemma}
\begin{proof}
\begin{align*}
    L_z^\top q &= q_1L_{z1} + \sum_{y\neq 1, z}L_{zy}q_y \\
    &\geq \frac{1}{2}L_{z1} + \sum_{y\neq 1, z}L_{zy}q_y \\
    &= \Big(\frac{1}{2} - \sum_{y\neq 1, z}q_y\Big)L_{z1} + \sum_{y\neq 1, z}(L_{z1} + L_{zy})q_y \\
    &\geq \Big(\frac{1}{2} - \sum_{y\neq 1, z}q_y\Big)L_{z1} + \sum_{y\neq 1, z}L_{1y}q_y \\
    &\geq L_{z1}q_z + \sum_{y\neq 1, z}L_{1y}q_y \\
    &= \sum_{y\neq 1}L_{1y}q_y = L_1^\top q.
\end{align*}
\end{proof}

\begin{lemma}\label{lemapp:ineqbayesrisks}
If $L$ is a distance, then $H_{\M} \leq 2H_{L}$.
\end{lemma}
\begin{proof}
In particular, $L$ is also symmetric. Recall that 
\begin{align}
    H_L(q) &= \min_{y\in\calY}L_y^\top q = \min_{\substack{a\succeq L p, \\ p\in\Delta.}}~a^\top q = \min_{a\in\calP_{L}}~a^\top q \\
    \frac{1}{2}H_{\M}(q) &= \min_{1a^\top + a1^\top \succeq L}~a^\top q = \min_{a\in\calP_{L}^{\M}}~a^\top q,
\end{align}
where the second expression is given by \cref{propapp:bayesrisks}. To show that $2H_{\M}\leq H_{L}$ we will show that $\calP_L\subseteq\calP_L^{\M}$. If $a\in\calP_L$, then there exists $p\in\Delta$ such that $a\succeq Lp$.
Moreover, if $L$ is a distance, it means that the triangle inequality $L(y, y') \leq L(y, z) + L(z, y')$, which is equivalent to $L \preceq 1L_z^\top + L_z1^\top$ for all $z\in\calZ$. This can also be written as
\begin{equation}
    L\preceq 1p^\top L + Lp1^\top, \hspace{0.5cm} \forall p\in\Delta.
\end{equation}
Finally, note that if $Lp\preceq a$, then $Lp1^\top \preceq a1^\top$, and the same holds for its transpose $1p^\top L \preceq 1a^\top$. Hence, we obtain that~$L\preceq 1a^\top + a1^\top$, which is equivalent to $a\in\calP_L^{\M}$.
\end{proof}

\begin{proposition}
If $L$ is a distance, then $\frac{1}{2}H_{\M}(q) = H_L(q)$, for all $q\in\Delta$ such that $\|q\|_{\infty} \geq 1/2$. Moreover, under this condition on $q$, it is calibrated with the argmax decoding.
\end{proposition}
\begin{proof}
Combining \cref{lemapp:distancecorners} and \cref{lemapp:ineqbayesrisks} gives 
\begin{equation*}
    H_{\M}(q) \leq 2H_{L}(q) = 2L_y^\top q,
\end{equation*}
for all $q\in\Delta$ such that $p_y\geq 1/2 \geq p_{z}$ for all $z\neq y$. Hence, in order to prove the equality at these dominant label points, we just need to find a matrix $Q\in U(q, q)$ such that $\langle L, Q\rangle_{\F} = 2L_y^\top q$.
We define the matrix $Q$ as
\begin{equation*}
    Q_{ij} = \left\{\begin{array}{lr}
        q_i & \text{if } (i=y)\wedge (i\neq j) \\
        q_j & \text{if }(j=y)\wedge (j\neq i) \\
        2q_{y}-1 & \text{if }i=j=y \\
        0 & \text{otherwise.}
    \end{array}\right. .
\end{equation*}
The matrix $Q$ has $q$ at the $y$-th row and $y$-th column and $2q_y - 1$ at the crossing point, instead of $2q_y$. The matrix is in $U(q, q)$ as the sum of the rows and columns gives $q$ and it is non-negative because $2q_y-1\geq 0$ by assumption. Moreover, the objective satisfies
\begin{align*}
    \langle L, Q\rangle_{\F} = \sum_{y\in\calY}~L_{y'}^\top Q_{y'} &= L_y^\top Q_y + \sum_{y'\neq y}L_{y'}Q_{y'} \\
    &= L_y^\top q + \sum_{y'\neq y}q_{y'}L_{y'}e_{y} \\ 
    &= L_y^\top q + \sum_{y'\neq y}q_{y'}L(y, y') = 2L_y^\top q.
\end{align*}
The first part of the result follows. For the second part, we show that $v_{\M}^\star(q) = -L_{y}$ at the points $q$ satisfying the assumption. Note that we have $H_{\M}(q) = 2L_y^\top q = S(v^\star(q))^\top q$, so we only need to show $S_{\M}(-L_{y}) = 2L_y$. For every $z\in\calY$, we have
\begin{align*}
    S_{\M}(-L_y, z) &= \max_{y'\in\calY}~L(z, y') + (-L_y)^\top e_{y'} - (-L_y)^\top e_z \\
    &= \max_{y'\in\calY}~\{L(z, y') - L(y, y')\} + L(y, z) \\
    &= 2L(y, z),
\end{align*}
where at the last step we have used $L(z, y') - L(y, y') \leq L(y, y')$ as $L$ is a distance and so the maximization is achieved at $y'=y$.
\end{proof}


\newpage
\section{PROOFS ON RESTRICTED-MAX-MARGIN LOSS}\label{secapp:C}
The following assumption \textbf{A1} will be key to prove our consistency results. \\
\textbf{Assumption A1:} If $q$ is an extreme point of $\Delta(y')$ for some $y'\in\calY$, then 
\begin{equation}
    \{q\in\Delta(y)\}\vee\{q_y = 0\}, \hspace{0.5cm} \forall y\in\calY.
\end{equation}
The following \cref{lem:intersection} will be useful for the results below.
\begin{lemma}\label{lem:intersection}
If \textbf{A1} is satisfied, then $\Delta(y)\cap\Delta(y')\neq\emptyset$ for all $y,y'\in\calY$.
\end{lemma}
\begin{proof}
If $\Delta(y)\cap\Delta(y')=\emptyset$ and \textbf{A1} is satisfied, then for every $q$ extreme point of $\Delta(y')$ we have that $q_y=0$. Hence, the prediction set $\Delta(y)$ is included in the non full-dimensional polyhedron~$\Delta\cap\{e_y=0\}$. As $\Delta = \cup_{y\in\calY}~\Delta(y)$, this implies that the point $e_{y'}\in\Delta(y')$ must be necessarily included in another $\Delta(z)$, which can only be possible if $L(z, y) = 0$. However, by assumption $L(z, y) = 0$ if and only if $z=y$.
\end{proof}

The following \cref{lem-th-5.1} shows that under \textbf{A1}, the Restricted-Max-Margin loss embeds the loss $L$, which in turn implies consistency. 
\begin{lemma}\label{lem-th-5.1} Assume \textbf{A1}. If $q$ is an extreme point of $\Delta(y)$ for some $y\in\calY$, then
\begin{equation*}
    qq^\top \in \underset{Q\in U(q, q)\cap\calC_L}{\argmax}\langle L, Q\rangle_{\F} \hspace{0.5cm}\text{and}\hspace{0.5cm}\Omega_{\MM}(q) = \Omega_{\RM}(q).
\end{equation*}
\end{lemma}
\begin{proof}
The matrix $qq^\top$ belongs to $U(q, q)$ as $qq^\top 1 = q$ and $qq^\top\succeq 0$ and to $\calC_L$ by assumption. Let $z\in y^\star(q)$. We have that 
\begin{align*}
    -\Omega_{\RM}(q) &= \underset{Q\in U(q, q)\cap\calC_L}{\max}\langle L, Q\rangle_{\F} \\
    &\geq \langle L, qq^\top\rangle_{\F} = \sum_{y}q_yL_y^\top q = \sum_yq_yL_z^\top q \\
    &= L_z^\top(1^\top qq^\top) = L_z^\top q = -\Omega_{\MM}(q)
\end{align*}
We have shown that $\Omega_{\RM}(q)\leq \Omega_{\MM}(q)$. Combining with Proposition 3.4 that states $\Omega_{\RM}\geq \Omega_{\MM}$, we obtain that~$\Omega_{\RM}(q) = \Omega_{\MM}(q)$.
\end{proof}
\begin{theorem}
If \textbf{A1} is satisfied, the Restricted-Max-Margin loss embeds~$L$ with embedding~$\psi(y) = -L_y$ and the loss is consistent to $L$ under the argmax decoding.
\end{theorem}
\begin{proof}
We split the proof into two parts.

\paragraph{First part: $\boldsymbol{S_{\RM}}$ embeds $\boldsymbol{L}$.} Let $z\in y^\star(q)$, so that $q\in\Delta(z)$. We can write $q$ as a convex combination of extreme points of the polytope $\Delta(z)$ as
\begin{equation*}
    q = \sum_{i=1}^m\alpha_iq_i,
\end{equation*}
where $\alpha\in\Delta_m$ and $q_i$ is an extreme point of $\Delta(z)$. The matrix $Q = \sum_{i=1}^m\alpha_iq_iq_i^\top$ belongs to $U(q, q)\cap \calC_L$ as:
\begin{itemize}
    \item $Q\in U(q, q)$: We have $Q1 = \sum_{i=1}^m\alpha_iq_iq_i^\top 1 = \sum_{i=1}^m\alpha_iq_i = q$, the same holds for $Q^\top$ and $\sum_{i=1}^m\alpha_iq_iq_i^\top \succeq 0$.
    \item $Q\in\calC_L$: For all $y\in\calY$, we have
    $(1L_y^\top - L)Q_y = \sum_{i=1}^m \alpha_i\underbrace{q_{i,y}(1L_y^\top - L)q_i}_{\preceq 0} \preceq 0$.
\end{itemize}
Moreover, we obtain:
\begin{align*}
    -\Omega_{\RM}(q) &= \underset{Q\in U(q, q)\cap\calC_L}{\max}\langle L, Q\rangle_{\F} \\
    &\geq \langle L, Q\rangle_{\F} = \sum_{i=1}^m\alpha_i\langle L, q_iq_i^\top\rangle_{\F} \\
    &= -\sum_{i=1}^m\alpha_i\Omega_{\RM}(q_i) =-\sum_{i=1}^m\alpha_i\Omega_{\MM}(q_i) \\
    &= \sum_{i=1}^m\alpha_iL_z^\top q_i = L_z^\top q = -\Omega_{\MM}(q)
\end{align*}
We have shown $\Omega_{\RM}\leq \Omega_{\MM}$. Combining with Proposition 3.4 that states $\Omega_{\RM}\geq \Omega_{\MM}$, we obtain $\Omega_{\RM} = \Omega_{\MM}$.

\paragraph{Second part: the embedding is $\boldsymbol{\psi(y) = -L_y}$.} By \cref{prop:bayesriskembedding}, we only need to show that $\psi(z) = -L_z$, i.e.,
\begin{equation*}
    S_{\RM}(-L_z, y) = \sup_{q\in\Delta(y)}~L_y^\top q + (-L_z)^\top q - (-L_z)^\top e_y = L(y, z),
\end{equation*}
which holds whenever
\begin{equation}
    \max_{q\in\Delta(y)}~(L_y - L_z)^\top q = 0,
\end{equation}
for all $y,z\in\calY$. Note that by construction $(L_y - L_z)^\top q \leq 0$ for all $q\in\Delta(y)$. Moreover, by \cref{lem:intersection} we have that $\Delta(z)\cap\Delta(y)\neq\emptyset$, so there exists $q\in\Delta(y)$ with $L_y^\top q = L_z^\top q$.

Finally, the argmax decoding is consistent as it is an inverse of the embedding $\psi(y) = -L_y$ as 
\begin{equation*}
d(\psi(y)) = \argmax_{y'\in\calY}~-L(y, y') = \argmin_{y'\in\calY}~L(y, y') = y.
\end{equation*}
\end{proof}

\begin{proposition}\label{prop:sufficientsimple} Assume that $q\in\Delta(y) \implies q_y > 0$ for all $q\in\Delta$. Then \textbf{A1} is satisfied. 
\end{proposition}
\begin{proof}
We will prove that if the Assumption is not satisfied then it exists a vertex of $\Delta(y)$ for some $y\in\calY$ such that $S\cap T\neq\emptyset$.
If the Assumption is not satisfied at vertex $q$, then $\{q\notin \Delta(y)\}\wedge\{q_y>0\}$, which means in particular that $S\cup T \subsetneq[k]$. This necessarily means that $S\cap T\neq\emptyset$ because we must have $|S|+|T|\geq k$ to have maximal rank as $q$ is a vertex. 
\end{proof}

\begin{proposition}
Consistency of the Max-Margin implies consistency of Restricted-Max-Margin. 
\end{proposition}
\begin{proof}
From \cref{propapp:maxmarginbayesrisks} and \cref{prop:weakerconsistency}, we know that if the Max-Margin loss is consistent to $L$, then the extreme points of the prediction sets $\Delta(y)$'s have to be of the form $1/2(e_y + e_{y'})$. We will see that in this case \textbf{(A1)} is always satisfied. Indeed, if $q$ is an extreme point of a prediction set, then is of the form $q=1/2(e_y + e_{y'})$, which satisfies $\{q\in\Delta(z)\}\vee\{q_z = 0\}$ for all $z\in\calY$, because $q\in\Delta(z)$ if $z\in\{y,y'\}$ and $q_z=0$ otherwise. 
\end{proof}

\end{document}